\documentclass[sigconf, nonacm]{acmart}

\AtBeginDocument{%
  }

\usepackage{tikz}
\usepackage{amsthm}
\usepackage{autobreak}
\allowdisplaybreaks
\usepackage{adjustbox}
\usepackage{booktabs}
\usepackage{makecell}
\usepackage{siunitx}
\usepackage{multirow}
\usepackage{balance}
\usepackage{longtable}
\usepackage{algorithm}
\usepackage[algo2e]{algorithm2e}
\usepackage{algorithmic}
\usepackage{bm}
\usepackage{graphicx}
\usepackage{placeins}
\usepackage{amsmath}

\usepackage{booktabs}
\usepackage{tabularx}
\usepackage{listings}
\usepackage{fix-cm}
\usepackage{hyperref}
\usepackage{wasysym}

\usepackage{amssymb}
\usepackage{array}
\usepackage{tabularx}
\usepackage{caption}
\usepackage{xcolor}
\usepackage{bm}
\usepackage{amsthm}

\usepackage{enumitem}
\usepackage{subcaption}

\usepackage[framemethod=tikz]{mdframed}
\usepackage{colortbl}
\usepackage{xcolor,soul}
\definecolor{lightyellow}{RGB}{255,255,153}
\definecolor{lightblue}{RGB}{153,255,255}

\newcommand{\pgftextcircled}[1]{
    \setbox0=\hbox{#1}%
    \dimen0\wd0%
    \divide\dimen0 by 2%
    \begin{tikzpicture}[baseline=(a.base)]%
        \useasboundingbox (-\the\dimen0,0pt) rectangle (\the\dimen0,1pt);
        \node[circle,draw,outer sep=0pt,inner sep=0.1ex] (a) {#1};
    \end{tikzpicture}
}

\sethlcolor{lightyellow}

\colorlet{highlightblue}{cyan!30}

\theoremstyle{definition}
\theoremstyle{remark}

\newtheoremstyle{bfnote}%
  {}{}
  {\itshape}{}
  {\bfseries}{.}
  { }{\thmname{#1}\thmnumber{ #2}\thmnote{ (#3)}}
\theoremstyle{bfnote}

\hypersetup{
  colorlinks,
  linkcolor={blue!70!green},
  citecolor={green!70!blue},
  urlcolor={orange!70!red}
}

\usepackage[most]{tcolorbox}
\usepackage{xcolor}
\usepackage{listings}

\lstdefinestyle{promptstyle}{
    basicstyle=\ttfamily\small,
    breaklines=true,
    columns=fullflexible,
    keepspaces=true,
    showstringspaces=false
}

\newcommand{\para}[1]{\noindent\textbf{{#1}}}
\def\ourmethod{DP-SelFT}

\begin{document}

\title{\ourmethod{}: Differentially Private Selective Fine-Tuning for Large Language Models} 
\begin{abstract}
Large language models (LLMs) are increasingly adapted to downstream tasks through fine-tuning rather than training task-specific models from scratch. However, fine-tuning data is often privacy-sensitive, and the resulting model may leak training information through attacks such as membership inference or data extraction. Differential privacy (DP) provides a principled way to mitigate such risks, but differentially private LLM fine-tuning often suffers from substantial utility degradation. This challenge is especially severe for large models, where gradient clipping and noise injection significantly distort optimization. To improve utility, prior work has combined DP with parameter-efficient fine-tuning (PEFT) methods such as LoRA, which reduce the trainable model size by constraining the \emph{form} of updates. However, the utility gap between DP and non-DP fine-tuning remains substantial.

In this work, we study a complementary direction: selection-based PEFT, which restricts \emph{where} updates are applied. Existing non-DP methods often use gradient- or loss-based criteria to identify parameters worth training, but extending this idea to DP faces three obstacles: selection itself consumes privacy budget, the required data-dependent estimates become noisy and unstable under DP perturbations, and non-DP selectors do not account for the damage that noisy updates may cause to the pre-trained model. To address these challenges, we propose \ourmethod{}, a framework for differentially private selective fine-tuning of LLMs. \ourmethod{} first performs lightweight differentially private synthetic data generation and then carries out selection on the resulting synthetic data, eliminating repeated privacy cost during selection. To improve stability, \ourmethod{} performs layer-level selection and ranks candidate layer subsets by validation performance on a synthetic validation split after temporary training on a synthetic training split. Most importantly, this temporary training is conducted under a perturbation regime matched to downstream DP fine-tuning, so that selection captures not only learning signal but also noise-induced damage, which is crucial in the DP setting. To further improve robustness to unseen perturbations, \ourmethod{} replaces random perturbations with worst-case perturbations of the same scale during selection. Intuitively, if a layer subset remains effective even under worst-case perturbation, it is more likely to remain robust under the stochastic perturbations encountered during actual DP fine-tuning. Extensive experiments on benchmark tasks show that \ourmethod{} consistently achieves a stronger privacy--utility trade-off than existing DP fine-tuning baselines under the same privacy guarantees.
\end{abstract}

\author{Haichao Sha}
\authornote{The first two authors contributed equally to this work.}
\affiliation{
  \institution{Renmin University of China}
  \city{Beijing}
  \country{China}
}
\email{sha@ruc.edu.cn}

\author{Zihao Wang}
\authornotemark[1]
\authornote{Correspondence to 
\href{mailto:zihao.wang@ntu.edu.sg}{zihao.wang@ntu.edu.sg} and 
\href{mailto:wei_dong@ntu.edu.sg}{wei\_dong@ntu.edu.sg}.}
\affiliation{
  \institution{Nanyang Technological University}
  \country{Singapore}
}
\email{zihao.wang@ntu.edu.sg}

\author{Yuncheng Wu}
\affiliation{
  \institution{Renmin University of China}
  \city{Beijing}
  \country{China}
}
\email{wuyuncheng@ruc.edu.cn}

\author{Hong Chen}
\affiliation{
  \institution{Renmin University of China}
  \city{Beijing}
  \country{China}
}
\email{chong@ruc.edu.cn}

\author{Wei Dong}
\authornotemark[2]
\affiliation{
  \institution{Nanyang Technological University}
  \country{Singapore}
}
\email{wei_dong@ntu.edu.sg}

\maketitle

\section{Introduction}

Large language models (LLMs) have become a standard foundation for modern AI systems. After large-scale pre-training, they are commonly adapted to downstream tasks through prompting or fine-tuning. While prompting is lightweight, fine-tuning remains essential for applications that require stronger specialization, better controllability, or domain-specific expertise.
However, the data used for fine-tuning is often sensitive, such as personal records, proprietary documents, or healthcare data. In such settings, sending data to remote APIs may be prohibited by regulatory, contractual, or organizational constraints~\cite{act1996health}. As a result, practitioners often need to fine-tune open-source foundation models locally, even when these models are less capable than commercial alternatives. This makes privacy-preserving local fine-tuning an important problem.

Differential privacy (DP)~\cite{dwork2006calibrating} provides a rigorous way to protect individual training examples. For deep learning, the standard approach is differentially private stochastic gradient descent (DP-SGD)~\cite{abadi2016deep}, which clips per-example gradients and adds calibrated Gaussian noise before each update. Adaptive private optimizers such as DP-AdamW~\cite{li2021large} have further improved optimization stability and empirical performance in LLM fine-tuning settings.

\para{Challenges.} Nevertheless, despite these advances, DP fine-tuning still often suffers from substantial utility degradation compared to non-private training.
A key reason is scale. LLMs contain a massive number of parameters, and useful task-specific updates are often distributed across a high-dimensional parameter space. Under DP, the optimization signal is distorted by gradient clipping and noise injection, making it harder to recover the fine-grained update directions needed for adaptation. At the same time, large pre-trained models can be surprisingly sensitive to parameter perturbations: even small noisy deviations may damage useful pre-trained structure and noticeably degrade downstream behavior. These difficulties naturally motivate \emph{parameter-efficient fine-tuning} (PEFT)~\cite{karimi2021compacter, houlsby2019parameter, hu2022lora}, which restricts training to a smaller set of parameters and thus reduces the effective dimension of the noisy optimization problem.

Existing work has explored integrating DP with non-selection-based PEFT methods such as LoRA~\cite{tsai2025differentially, yu2021differentially}. These methods constrain the \emph{form} of the update, for example by learning low-rank adaptations instead of updating the full parameter space. This structured parameterization indeed improves the privacy--utility trade-off relative to full-parameter DP fine-tuning. However, the gap between DP and non-DP fine-tuning remains substantial even with such designs, motivating the need for more effective PEFT schemes.

A complementary line of work in non-private settings studies \emph{selection-based} efficient fine-tuning, which automatically identifies and updates only the most important parameters for a given task~\cite{zhang2024gradient, he2023sensitivity}. This direction is orthogonal to non-selection-based methods such as LoRA: the latter constrains \emph{how} updates are parameterized, while the former constrains \emph{where} updates are applied. Despite its strong promise, this direction remains largely unexplored in differentially private LLM fine-tuning.

A natural idea is to directly adapt existing selection methods to the DP setting. However, we find that this is far from straightforward. First, \emph{selection itself consumes privacy budget}. Most existing methods rely on data-dependent statistics such as gradients or losses to estimate parameter importance~\cite{zhang2024gradient, he2023sensitivity}. Under DP, these statistics must be privatized, which either incurs additional privacy cost or produces highly noisy estimates. Second, \emph{importance estimation becomes unstable under noise}. Selection is fundamentally a ranking problem, and DP noise can easily distort the relative order among candidates, especially in the high-dimensional regime of LLMs. Third, and most fundamentally, there is an \emph{objective mismatch} between standard selection criteria and DP fine-tuning. Classical selectors aim to find parameters that are most useful for reducing loss under clean optimization. Under DP, however, the actual updates are clipped and noisy. A parameter subset that looks beneficial in a clean setting may be a poor choice if noisy updates on that subset cause disproportionate damage to the pre-trained model.

\para{Our solution.} In this work, we propose \ourmethod{}, a framework for differentially private selective fine-tuning of LLMs. Our design is built around observations corresponding to the above challenges. 

\textit{First}, to avoid repeated privacy expenditure during selection, we decouple selection from the private dataset through a lightweight differentially private synthetic-data construction stage. The resulting synthetic data is then treated as public for downstream selection, so the selection process itself incurs no additional privacy cost. We further split this synthetic dataset into a training split and a validation split. For each candidate layer subset, we first perform temporary training on the synthetic training split, and then evaluate the resulting model on the validation split. The validation performance is used as the criterion for layer selection.

\textit{Second}, instead of selecting individual parameters, we perform selection at the layer level. This coarser granularity yields substantially more stable estimates and improves transferability from synthetic data to the real private task. 

\textit{Third}, and most importantly, we argue that selection in the DP setting should be carried out under the same perturbation regime as the eventual private fine-tuning stage, rather than under clean optimization. In particular, candidate subsets should be evaluated in the presence of both gradient clipping and noise injection, since these operations fundamentally alter the update dynamics and determine which parts of the model can be trained reliably under privacy constraints.
Building on this principle, we introduce worst-case perturbation into the selection stage. For each candidate layer subset, \ourmethod{} performs temporary training on the synthetic training split under a perturbation-aware update rule, and then uses validation performance on the synthetic validation split as the selection criterion. Instead of evaluating a subset under one sampled random perturbation, \ourmethod{} evaluates it under an adversarial perturbation with the same scale as the target DP noise.
This design reflects our central insight: under DP, a good trainable subset should be not only \emph{learnable}, but also robust to noisy updates. Worst-case perturbation provides a conservative stress test, favoring layer subsets that remain effective under unfavorable perturbations and are therefore more likely to generalize to the stochastic noise encountered during actual DP fine-tuning.

\para{Empirical results.}
We conduct extensive experiments on four foundation models across four GLUE benchmark datasets~\cite{GLUE}, covering sentiment analysis, natural language inference, and paraphrase detection, under three privacy budgets. The results show that \ourmethod{} improves downstream utility over five baselines, including full-parameter fine-tuning, existing DP fine-tuning methods~\cite{karimi2021compacter, houlsby2019parameter, hu2022lora}, and manually designed heuristic layer selection~\cite{liu2023differentially}, while preserving the same privacy guarantees. For example, at $\varepsilon=5$, \ourmethod{} improves the average accuracy of OPT-350M from $64.99\%$ to $67.34\%$, yielding a gain of $2.35\%$.
Moreover, ablation studies show that naive adaptations of non-private selective fine-tuning methods are suboptimal, confirming the need for a DP-specific design. Further ablations and sensitivity analyses validate each component of \ourmethod{} and show stable performance across different numbers of selected layers, privacy budget allocations, and model sizes.

Finally, we provide theoretical analysis (see \autoref{subsec:analysis}) showing that \ourmethod{} improves the privacy--utility trade-off by balancing two DP-specific forces: preserving useful learning signal and reducing noise-induced model damage. Our analysis further justifies the use of worst-case perturbation as a robust selection criterion that better generalizes to unseen DP noise.

\para{Contributions.}
Our key contributions are summarized as follows:

\noindent$\bullet$\textit{~Parameter selection for DP fine-tuning.}
We identify selection-based PEFT as a promising but underexplored direction for differentially private LLM fine-tuning. Unlike LoRA-style methods that constrain the \emph{form} of updates, parameter selection constrains \emph{where} updates are applied. Building on this observation, we propose \ourmethod{}, a DP-specific selective fine-tuning framework that selects trainable layers under privacy and noise constraints.

\noindent$\bullet$\textit{~Theoretical understandings.}
We provide theoretical analysis explaining why selective fine-tuning can improve the DP privacy--utility trade-off. Our results show that layer selection induces a trade-off between retained learning signal and reduced DP noise damage, and that worst-case perturbation yields a more robust selection criterion for unseen DP noise.

\noindent$\bullet$\textit{~Extensive empirical studies.}
We conduct comprehensive experiments across multiple models, tasks, and privacy budgets. The results show that \ourmethod{} consistently improves utility over existing DP fine-tuning and heuristic layer-selection baselines under the same privacy guarantees. Ablation and sensitivity studies further validate the contribution of each design component.
\section{Background and Related Work}
\label{sec:back}

\subsection{LLM Fine-tuning}
\label{subsec:LLM Fine-tuning}

Large language models (LLMs) have become a core building block for a wide range of applications. In practice, these applications are often built on top of a foundation model and then adapted to downstream tasks. Two common adaptation strategies are in-context learning through prompting and parameter updating through fine-tuning. Prompting is lightweight and easy to deploy, while fine-tuning is often preferred when developers need stronger task specialization, better controllability, or domain-specific expertise.

However, the current LLM ecosystem presents an important practical constraint. The most capable models are typically proprietary and only accessible through remote APIs. By contrast, open-source models that can be deployed locally are often weaker in capability. Moreover, even when model providers expose fine-tuning APIs, such support is usually limited to only a subset of their models. In many cases, newer or more capable models are available for inference only and cannot be directly fine-tuned by users.

This limitation is further compounded by privacy and trust concerns. In many real-world scenarios, the training data contains sensitive or proprietary information (e.g., user data, enterprise documents, or medical records). Sending such data to external APIs (either through prompting or fine-tuning) introduces significant privacy risks and may violate regulatory or organizational constraints. As a result, practitioners often resort to locally fine-tuning open-source models, even at the cost of reduced model quality, in order to maintain full control over data and training processes.

This trade-off between model capability and data privacy has made local fine-tuning an important practical paradigm. It also motivates the study of fine-tuning methods under realistic privacy and security constraints.

\subsection{Parameter-Efficient Fine-tuning}
\label{subsec:PEFT}

Building on the practical constraints discussed above, parameter-efficient fine-tuning (PEFT) has emerged as a key paradigm for adapting large models under limited computational and data budgets. Instead of updating all model parameters, PEFT methods restrict training to a small subset of parameters or introduce lightweight trainable components, thereby significantly reducing memory, computation, and storage overhead.
Existing PEFT approaches can be broadly categorized into two orthogonal directions: \emph{non-selection-based} methods and \emph{selection-based} methods. The former introduces structured parameterizations to reduce the number of trainable parameters, while the latter explicitly selects a subset of existing parameters to update. These two directions are largely complementary and can often be combined in practice.

\para{Non-selection-based methods.}
A prominent line of work focuses on modifying the parameterization of fine-tuning without explicitly selecting parameters.
Low-Rank Adaptation (LoRA)~\cite{hu2022lora} models parameter updates as low-rank decompositions added to existing weight matrices, effectively reducing the number of trainable parameters while preserving expressiveness.
Adapter-based methods, introduced by Houlsby et al.~\cite{houlsby2019parameter}, insert small bottleneck layers with residual connections after attention and feed-forward modules, enabling task-specific adaptation without modifying the backbone parameters.
Subsequent work further improves parameter efficiency; for example, Compacter~\cite{karimi2021compacter} replaces dense adapter matrices with tensorized parameter sharing, significantly reducing the number of trainable parameters while maintaining performance.

\para{Selection-based methods.}
Another line of work aims to identify and update only the most important parameters for a given task.
Gradient-based selection methods~\cite{zhang2024gradient} prioritize parameters with large gradients, under the intuition that these parameters contribute most to the loss landscape and thus benefit more from adaptation.
Similarly, loss-based selection approaches~\cite{he2023sensitivity} evaluate the impact of tuning individual parameters or modules by measuring the resulting loss reduction, selecting those that yield the largest loss reduction.

These two paradigms capture complementary aspects of efficient adaptation: non-selection-based methods constrain the \emph{form} of updates, while selection-based methods constrain the \emph{location} of updates. As a result, hybrid approaches (e.g., applying LoRA only to selected layers or parameters) have been explored to further improve efficiency. This design space highlights an important trade-off between parameter efficiency, optimization flexibility, and downstream performance, which remains an active area of research.

\subsection{Differentially Private Machine Learning}
\label{subsec:DPML}

\textit{Differential privacy}~\cite{dwork2006calibrating} (DP) provides a rigorous framework for limiting the leakage of individual-level information in data analysis and machine learning. At a high level, DP ensures that the output of a randomized mechanism does not depend significantly on any single data record. This guarantees that an adversary observing the output cannot reliably infer whether a particular individual’s data was included in the training set.

Formally, DP is defined as follows:

\begin{definition}[($\varepsilon$,$\delta$)-Differential Privacy]
Given two neighboring datasets $D$ and $D'$ differing by one record, a randomized mechanism $\mathcal{M}$ is said to satisfy ($\varepsilon$,$\delta$)-differential privacy if for all measurable sets $S$,
$$
Pr[\mathcal{M}(D)\in S] \leq e^{\varepsilon } \cdot Pr[\mathcal{M}(D') \in S] + \delta,
$$
where $\varepsilon$ is the privacy budget and $\delta$ is a small failure probability.
\end{definition}

A smaller $\varepsilon$ implies stronger privacy, as it bounds how much the output distribution can change with respect to any individual record.

\para{Key properties.}
Two properties make DP particularly suitable for machine learning.
First, \emph{sequential composition} states that privacy loss accumulates across multiple accesses to the same dataset: if $\mathcal{M}_1, \ldots, \mathcal{M}_k$ are applied sequentially with budgets ${\varepsilon_1, \ldots, \varepsilon_k}$, the overall mechanism satisfies $\varepsilon = \sum_i \varepsilon_i$.
Second, \emph{post-processing invariance} ensures that any transformation applied to the output of a DP mechanism does not incur additional privacy cost. These properties allow DP guarantees to be tracked modularly throughout complex training pipelines.

\para{DP-SGD.}
The standard approach for training deep models under differential privacy~\cite{wei2023dpmlbench} is \emph{differentially private stochastic gradient descent} (DP-SGD)~\cite{abadi2016deep}. At each iteration, given a mini-batch $\mathcal{B}$, per-example gradients $g_i = \nabla_{\theta} \mathcal{L}(\theta; x_i)$ are first clipped to bound their sensitivity:
\begin{equation}
\tilde{g}_i = \frac{g_i}{\max\left(1, \frac{|g_i|*2}{C}\right)},
\end{equation}
where $C$ is the clipping norm. The clipped gradients are then aggregated with Gaussian noise:
\begin{equation}
\bar{g} = \frac{1}{|\mathcal{B}|} \left( \sum*{i \in \mathcal{B}} \tilde{g}_i + \mathcal{N}(0, \sigma^2 C^2 \mathbf{I}) \right),
\end{equation}
followed by a standard parameter update $\theta \leftarrow \theta - \eta \bar{g}$. The clipping step limits the influence of any individual sample, while the injected noise ensures indistinguishability between neighboring datasets.

To track the cumulative privacy loss over multiple iterations, modern implementations adopt \emph{R'enyi Differential Privacy} (RDP)~\cite{mironov2019r}. A mechanism $\mathcal{M}$ satisfies $(\alpha, \varepsilon)$-RDP if for any neighboring datasets $D, D'$,
\begin{equation}
D_{\alpha}(\mathcal{M}(D),|,\mathcal{M}(D')) \leq \varepsilon,
\end{equation}
where $D_{\alpha}(\cdot|\cdot)$ is the R'enyi divergence of order $\alpha>1$. RDP composes additively across iterations, enabling tight accounting of the total privacy cost, which can be converted back to an $(\varepsilon,\delta)$ guarantee. This tighter analysis is particularly important for iterative training procedures such as DP-SGD.

\para{DP-AdamW.}
Recent work extends adaptive optimizers such as AdamW to the private setting, often referred to as DP-AdamW~\cite{li2021large}. DP-AdamW retains gradient clipping and noise injection while incorporating momentum and adaptive learning rates, which improves convergence and stability for large-scale DP fine-tuning.

\para{DP and parameter-efficient fine-tuning.}
Applying DP to LLM fine-tuning remains challenging due to the large number of trainable parameters and the sensitivity of gradients. PEFT offers a natural direction to mitigate these issues by reducing the effective parameter space. Existing work has primarily explored integrating DP with non-selection-based methods such as LoRA~\cite{tsai2025differentially, yu2021differentially}, leveraging their reduced parameterization to improve utility under noise. In contrast, the combination of DP with selection-based methods remains largely under-explored. In particular, systematically identifying which parameters should be updated under DP constraints (where both gradient noise and privacy budget interact) poses a non-trivial design challenge.

\subsection{Threat Model}
\label{subsec:threat}

We consider a setting where a defender aims to fine-tune a language model on sensitive data while providing formal privacy guarantees.

\para{Defender capabilities.}
The defender has access to (i) a locally deployable open-source foundation model (e.g., from HuggingFace), and (ii) general-purpose commercial LLM APIs (e.g., \texttt{GPT-4o}) for auxiliary usage. The defender can perform full training and fine-tuning locally on the open-source model, and may use external APIs for non-sensitive processing or general-purpose tasks.

\para{Privacy constraints.}
The training data contains sensitive information (e.g., user data or proprietary documents) that must not leave the local environment. As a result, the defender cannot directly use API-based fine-tuning services, nor can they expose raw data through prompting. All operations involving private data must be performed locally with formal privacy guarantees.

This assumption is realistic in many practical scenarios where data is subject to strict regulatory or organizational constraints (e.g., healthcare, finance, or enterprise settings), which prohibit transmitting raw data to third-party services.

\para{Learning objective.}
The defender aims to fine-tune the open-source model under differential privacy, ensuring that the resulting model satisfies a target $(\varepsilon,\delta)$-DP guarantee. At the same time, the defender seeks to maximize downstream utility despite the performance degradation introduced by DP mechanisms such as gradient clipping and noise injection.

\section{Differentially Private LLM Fine-tuning}
\label{sec:observation}

\subsection{Motivation}

Differentially private fine-tuning of LLMs often exhibits a clear utility gap compared to non-private fine-tuning. While this challenge exists more broadly in deep learning, it becomes particularly pronounced in LLMs due to their scale and sensitivity.

One reason is that DP training fundamentally perturbs the optimization process. In DP-SGD, each update is computed from clipped per-example gradients with additional Gaussian noise. For small models, such perturbations may still allow the optimizer to follow a relatively stable descent direction. For LLMs, however, the parameter space is extremely high-dimensional, and useful task-specific updates are often distributed across many parameters. As a result, the signal carried by each update is already weak and diffuse, making it more vulnerable to distortion after clipping and noise injection. The larger the model, the harder it becomes for a noisy optimizer to consistently recover the fine-grained update directions needed for adaptation.

Another reason is that modern pre-trained language models are not uniformly robust to parameter perturbations. Although they contain substantial redundancy, their downstream behavior can still be highly sensitive to changes in certain layers or subspaces. Prior observations in model robustness and parameter perturbation suggest that even relatively small weight changes can noticeably alter model behavior~\cite{wang2024tossing, hong2019terminal}. Under DP, this issue becomes more severe because the perturbation is not adversarially optimized or task-aware; instead, it is injected indiscriminately as part of the privacy mechanism. Consequently, DP fine-tuning must operate under a stricter constraint: it must adapt the model to a new task while simultaneously avoiding excessive damage to the useful structure already encoded in the pre-trained weights.

These observations suggest that the difficulty of DP LLM fine-tuning is not merely that ``noise is added''. Rather, the key issue is that privacy noise interacts poorly with the scale and fragility of large pre-trained models. This naturally motivates restricting the trainable space so that optimization is concentrated on a smaller and more controllable subset of parameters.

Parameter-efficient fine-tuning (PEFT) provides such a mechanism. By reducing the number of trainable parameters, PEFT narrows the space in which noisy optimization takes place. This can help in at least two ways. First, it reduces the effective complexity of adaptation, making it easier for the optimizer to extract a useful signal from noisy gradients. Second, it limits the portion of the model directly exposed to noisy updates, which can preserve more of the original pre-trained knowledge. Existing work has therefore explored combining DP with non-selection-based PEFT methods such as LoRA~\cite{tsai2025differentially, yu2021differentially}, where updates are constrained to low-rank adapters rather than the full parameter space.

This design is intuitively appealing because LoRA imposes structure on the update itself. Instead of learning an unrestricted perturbation to each weight matrix, it learns a low-rank update, which acts as an implicit regularizer. Under DP noise, this structured parameterization can be more stable than full fine-tuning, since the optimizer only needs to recover a smaller number of trainable directions. Empirically, this indeed improves the privacy--utility trade-off relative to full-parameter DP fine-tuning~\cite{yu2021differentially}.

However, existing PEFT solutions are still insufficient. Even under LoRA-based DP fine-tuning, prior work reports non-trivial utility drops, including several-point degradations on benchmarks such as GLUE~\cite{GLUE}. On such benchmarks, a drop of a few percentage points is already substantial, especially when the remaining performance gap between methods is often narrow. This suggests that merely constraining the \emph{form} of the update is not enough. There remains room to better control \emph{where} the noisy updates are applied.

This motivates us to explore more effective PEFT strategies for DP fine-tuning, to further improving the privacy--utility trade-off.

\subsection{Challenges}
\label{subsec:challenge}

A natural next step is to consider parameter selection. In non-private settings, a rich line of work has shown that not all parameters are equally worth updating: some contribute more to downstream adaptation, while others can remain frozen with little performance loss. This suggests an appealing possibility for DP fine-tuning: if one could identify a particularly suitable subset of parameters, then noisy updates could be concentrated on that subset, potentially yielding better utility with the same privacy budget.

At first glance, existing selection-based methods seem directly applicable here. For example, gradient-based methods~\cite{zhang2024gradient} prioritize parameters with large gradients, based on the intuition that these parameters are currently most relevant to reducing the objective. Loss-based methods~\cite{he2023sensitivity} instead evaluate how much improvement is obtained when certain parameters are updated, and select those that appear most beneficial. In standard fine-tuning, such criteria are often reasonable proxies for learnability.

However, directly transplanting these methods into the DP setting introduces several complications.

\para{Privacy cost of selection.}
The first challenge is that selection itself is a data-dependent procedure. To decide which parameters are important, the algorithm must query the private training data to estimate gradients, losses, or sensitivities. Under DP, such queries are not free: they consume privacy budget just like training updates do. This creates an immediate tension. If substantial privacy budget is spent on accurately identifying good parameters, less budget remains for the actual fine-tuning stage. Conversely, if only a very small budget is allocated to selection, the resulting estimates may be too noisy to be useful. Therefore, unlike in non-private settings, selection cannot be treated as a cheap preprocessing step.

This issue is especially severe in LLMs, where the candidate parameter space is enormous. Estimating importance over such a large space is already statistically difficult, and DP noise makes it harder still. In practice, this means that selection may only be performed once, or at most a few times, under a limited privacy budget. As a result, the estimate must be both cheap and reliable, which existing methods are not designed for.

\para{Noisy and unstable importance estimation.}
A second challenge is that the signals used by classical selection methods become much less reliable under DP noise. Gradient magnitudes, loss reductions, and sensitivity scores are all fine-grained quantities. In non-private optimization, they can already be noisy due to mini-batch variation. Under DP, they must additionally be clipped and perturbed. This can substantially distort the ranking of parameters, especially when many candidate parameters have similar importance. In such cases, the noise may dominate the true differences, causing the selected subset to be unstable or even systematically misleading.

This problem is not merely one of reduced precision. Selection is fundamentally a ranking problem: what matters is not only estimating each parameter well in isolation, but also preserving the relative order among many parameters. DP noise can easily disrupt this ordering, particularly in the high-dimensional regime of LLMs. Therefore, a naive DP version of an existing selector may fail even if its non-private counterpart works well.

\para{Objective mismatch in the DP setting.}
More fundamentally, existing selection methods optimize the wrong objective for DP fine-tuning. In non-private settings, it is natural to select parameters that appear most helpful for reducing loss, since updates can be applied relatively faithfully. Under DP, however, the actual update is a noisy approximation of the desired direction. Therefore, a parameter that looks highly useful in principle may still be a poor choice if its update is especially fragile to clipping and noise.

This introduces a DP-specific objective mismatch. Under privacy constraints, a good parameter subset should not only be \emph{learnable}, in the sense that updating it can improve task performance, but also \emph{robust}, in the sense that the model can tolerate noisy updates on that subset without incurring disproportionate damage. These two properties are related but not identical. A parameter may have a large gradient because it is highly task-relevant, but that same sensitivity may also mean that noisy perturbations on it are particularly destructive. Conversely, some parameters may support smaller but more stable improvements under noise.

This distinction is largely absent in prior selection-based fine-tuning, because standard training does not need to account for systematic privacy noise. In the DP setting, by contrast, preserving the useful structure of the pre-trained model becomes part of the optimization problem itself. The goal is no longer just to identify parameters worth learning, but to identify parameters worth learning \emph{under noisy and privacy-constrained updates}.

\subsection{Key Insights}
\label{subsec:insight}

The challenges above suggest that an effective selection strategy for DP LLM fine-tuning should satisfy three requirements simultaneously: it should incur little additional privacy cost, produce stable importance estimates, and align the selection criterion with the noisy optimization dynamics of DP training. Our design is guided by the following insights.

\para{Insight 1: Decoupling selection from private data via DP synthetic data.}
The main obstacle in applying selection-based methods under DP is that selection itself is data-dependent and therefore consumes privacy budget. A natural way to break this dependency is to first construct a privacy-preserving surrogate dataset, and then perform selection entirely on that surrogate.

Concretely, we allocate a small portion of the overall privacy budget to \emph{differentially private synthetic data generation}~\cite{xie2024differentially, yue2023synthetic, kurakin2023harnessing, mattern2022differentially}, a well-studied problem with existing practical solutions. Once generated, the synthetic data can be treated as public for all subsequent stages, meaning that any computation performed on it incurs no additional privacy cost. We can therefore use part of the synthetic data for training-time analysis and reserve another part as validation data for parameter selection.

The key intuition is that this synthetic data does not need to be of sufficiently high quality to support full downstream training. Instead, it only needs to preserve enough coarse task structure to distinguish relatively better parameter subsets from worse ones. Empirically, selection decisions are often more robust than final model accuracy to moderate data imperfections: even if the synthetic data does not faithfully reproduce the full private distribution, it may still preserve the broad layer-wise adaptation patterns needed for selection. This suggests that a relatively small privacy budget may already suffice for generating useful synthetic data for this purpose.

\para{Insight 2: Coarsening the selection granularity improves stability.}
Even with a surrogate dataset, directly estimating importance at the parameter level remains unstable. Fine-grained importance scores are highly sensitive to both data and optimization noise, and small estimation errors can easily change the resulting ranking among millions or billions of candidate parameters.

This motivates reducing the granularity of selection, from \emph{parameter selection} to \emph{layer selection}. Compared with individual parameters, layers are much larger and more semantically coherent units. Their importance is easier to estimate reliably, and their ranking is less sensitive to small fluctuations in the data. Moreover, layer-level preferences are more likely to transfer from synthetic data to the real private data distribution, since higher-level adaptation patterns are typically more stable than fine-grained parameter-wise signals.

In other words, coarser selection sacrifices some flexibility, but gains substantial robustness. This trade-off is favorable: a slightly less expressive but much more reliable selector can be more useful than a highly flexible selector whose estimates are unstable.

\para{Insight 3: Selection should account for noisy updates explicitly.}
Existing selection methods typically evaluate importance under clean optimization, implicitly assuming that the chosen parameters will later receive accurate updates. This assumption no longer holds under DP, where every update is clipped and perturbed by noise. As discussed earlier, this creates an objective mismatch: parameters that appear most beneficial in a clean setting may not be the best ones once the actual training dynamics become noisy.

To address this mismatch, selection should be performed under the same perturbation regime as the eventual DP fine-tuning stage. That is, instead of evaluating candidate layers using clean updates, we inject noise and apply clipping during the selection-time training process and choose the layers that yield the best validation performance under noisy optimization. In this way, the selector explicitly captures not only how much a layer can help learning, but also how robust that layer is to noisy updates.

To further improve transferability, we replace random noise with \emph{worst-case noise} of the same scale during selection. The intuition is that random noise provides only one realization of the perturbation that may be encountered during DP training, whereas worst-case perturbations approximate a more conservative stress test of robustness. If a layer continues to perform well under such adverse perturbations, it is more likely to remain stable across different realizations of DP noise in the actual private fine-tuning stage. This yields a selector that is less tied to a particular random draw and better aligned with the robustness requirements of DP learning.

\begin{figure*}[!t]
\centerline{\includegraphics[width=.9\linewidth]{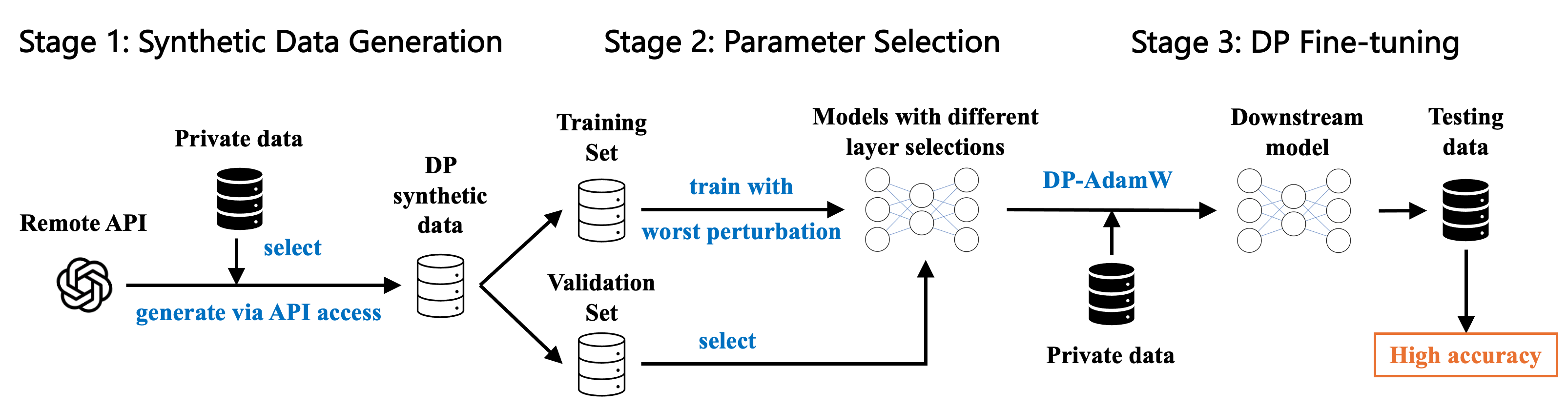}}
\caption{Overview of \ourmethod{}.}
\Description{}
\label{fig:overview}
\end{figure*}

\section{Design of \ourmethod{}}
\label{sec:method}

\subsection{Overview}

In this section, we present the design of \ourmethod{}. \autoref{fig:overview} illustrates the overall pipeline, which consists of three stages. For ease of exposition, we describe the method under a classification task, although the design is more general.

\para{(1) Synthetic data generation.}
As discussed in \autoref{subsec:LLM Fine-tuning}, remote commercial APIs typically provide stronger generation capability than locally deployable open-source models. Inspired by \cite{xie2024differentially}, we first leverage such remote APIs to generate a large pool of \emph{candidate} data samples using only simple, task-relevant prompts abstracted from the private task, such as ``write a movie review''. Although these prompts are derived from a high-level understanding of the private task, they remain sufficiently abstract and contain no dataset-specific or individual-level information. As a result, this stage does not reveal private data and incurs no privacy cost.

Given this unlabeled candidate dataset, we then allocate a small portion of the privacy budget to adapt it to the private task locally. Specifically, we use the private dataset to identify candidate samples that are closer to the target data distribution and to assign them labels in a differentially private manner. This yields a labeled synthetic dataset that better matches the private task, while only consuming a limited privacy budget. Intuitively, the required budget can remain small because this synthetic dataset is used only for selection rather than real model fine-tuning.

\para{(2) Parameter selection.}
We further split the selected labeled synthetic dataset into a training set and a validation set. We then evaluate different layer selections using this synthetic data. For each candidate layer subset, we train a temporary model on the synthetic training split and evaluate its performance on the synthetic validation split. Importantly, this temporary training is performed under the same perturbation regime as the eventual DP fine-tuning stage, so that the selection criterion reflects the actual optimization conditions under privacy constraints.

Moreover, instead of injecting random noise during this stage, we use \emph{worst-case noise} with the same scale. This provides a more conservative estimate of robustness and reduces the risk that the selected layers are overly tuned to a particular random noise realization. The validation performance under such perturbation-aware training is then used as the criterion for layer selection.

\para{(3) DP fine-tuning.}
After identifying the selected layers, we perform the final training on the private dataset using DP-AdamW, updating only the chosen layers to obtain the final model. Since \ourmethod{} determines \emph{where} to apply the updates rather than changing the update parameterization itself, it is naturally compatible with non-selection-based PEFT methods such as LoRA. In this sense, \ourmethod{} can be used either as a standalone layer selection strategy or as a complementary component to further improve existing PEFT-based DP fine-tuning pipelines.

\subsection{Synthetic Data Generation}
\label{subsec:synthetic_generation}

In this section, we provide the design details of the synthetic data generation stage in \ourmethod{}.

We construct the synthetic dataset in two stages. Let $\mathcal{D}_{\mathrm{pri}}=\{(x_i,y_i)\}_{i=1}^{n}$ denote the private dataset. We first generate an unlabeled candidate pool
\begin{equation}
\mathcal{C}=\{\tilde{x}_j\}_{j=1}^{m}
\end{equation}
using a remote commercial API queried with a task-level prompt template $\pi$. Here, $\pi$ captures only the high-level task semantics (e.g., review writing or question answering) and contains no private examples or dataset-specific content. Since $\pi$ is a task abstraction rather than a data-dependent query, this stage is independent of $\mathcal{D}_{\mathrm{pri}}$ and incurs no privacy cost.

In the second stage, we use the private dataset only to \emph{select} and \emph{label} a subset of candidates in a differentially private manner. To do so, we map both private and candidate samples into a local embedding space using a public encoder $\phi(\cdot)$ executed entirely on the defender side.

For each private example $x_i$, we identify its nearest candidate
\begin{equation}
j^*(i)=\arg\min_{j\in[m]} d\!\left(\phi(x_i),\phi(\tilde{x}_j)\right),
\end{equation}
where $d(\cdot,\cdot)$ is a distance metric such as Euclidean distance or cosine distance. Each private sample then casts one vote to its nearest candidate. Aggregating these votes gives a histogram over candidates,
\begin{equation}
h_j=\sum_{i=1}^{n}\mathbf{1}\{j^*(i)=j\}, \qquad j=1,\dots,m.
\end{equation}

Since each private example contributes to exactly one bin, the histogram has sensitivity $1$ at the sample level. We therefore privatize it by adding independent Gaussian noise to each bin:
\begin{equation}
\tilde{h}_j = h_j + z_j, \qquad z_j \sim \mathcal{N}(0,\sigma^2).
\end{equation}
The noisy histogram $\tilde{h}$ is then used to rank candidates, and we retain the top-$k$ samples as the synthetic subset. Intuitively, candidates receiving more votes are closer to the private data distribution, while the added noise ensures that the influence of any single private example is bounded.

After selection, we assign labels to the retained candidates using the private labels in a similarly aggregated manner. Concretely, for each selected candidate $\tilde{x}_j$ and each class $c\in\mathcal{Y}$, we count how many private samples of class $c$ vote for $\tilde{x}_j$:
\begin{equation}
h_{j,c}=\sum_{i=1}^{n}\mathbf{1}\{j^*(i)=j,\; y_i=c\}.
\end{equation}
We then add Gaussian noise to these class-wise counts,
\begin{equation}
\tilde{h}_{j,c}=h_{j,c}+z_{j,c}, \qquad z_{j,c}\sim\mathcal{N}(0,\sigma_{\mathrm{lab}}^2),
\end{equation}
and assign the synthetic label by noisy majority vote,
\begin{equation}
\tilde{y}_j=\arg\max_{c\in\mathcal{Y}} \tilde{h}_{j,c}.
\end{equation}
This produces a labeled synthetic dataset
\begin{equation}
\widetilde{\mathcal{D}}_{\mathrm{syn}}=\{(\tilde{x}_j,\tilde{y}_j)\}_{j\in\mathcal{S}},
\end{equation}
where $\mathcal{S}$ is the index set of selected candidates.

Our key design choice is to spend privacy budget only on this lightweight selection-and-labeling step, rather than on synthetic data generation itself or repeated parameter selection on the private dataset. This keeps the additional privacy cost small while still adapting the candidate pool toward the private task distribution. Importantly, the resulting synthetic data is not required to support high-quality standalone training. Instead, it only needs to preserve enough task-relevant structure for reliable layer selection, which is empirically much more tolerant to moderate data imperfections.

\subsection{Parameter Selection}
\label{subsec:selection}

In this section, we present the design details of the parameter selection stage in \ourmethod{}.

Given the labeled synthetic dataset
\begin{equation}
\widetilde{\mathcal{D}}_{\mathrm{syn}}=\{(\tilde{x}_j,\tilde{y}_j)\}_{j\in\mathcal{S}},
\end{equation}
we split it into a synthetic training set and a synthetic validation set,
\begin{equation}
\widetilde{\mathcal{D}}_{\mathrm{syn}}^{\mathrm{tr}}
\cup
\widetilde{\mathcal{D}}_{\mathrm{syn}}^{\mathrm{val}}
=
\widetilde{\mathcal{D}}_{\mathrm{syn}},
\qquad
\widetilde{\mathcal{D}}_{\mathrm{syn}}^{\mathrm{tr}}
\cap
\widetilde{\mathcal{D}}_{\mathrm{syn}}^{\mathrm{val}}
=\emptyset.
\end{equation}

Let the pre-trained model be parameterized by
\begin{equation}
\theta^{(0)}=\{\theta^{(0)}_1,\dots,\theta^{(0)}_L\},
\end{equation}
where $\theta^{(0)}_\ell$ denotes the parameters of the $\ell$-th layer and $L$ is the total number of layers. Our goal is to select a subset of layers
\begin{equation}
\Lambda \subseteq [L]
\end{equation}
for subsequent DP fine-tuning on the private dataset.

For each candidate layer subset $\Lambda$, we construct a temporary model by allowing only the layers in $\Lambda$ to be updated while freezing all remaining layers. Formally, the trainable parameter set is
\begin{equation}
\theta_\Lambda=\{\theta_\ell:\ell\in\Lambda\},
\end{equation}
while $\theta_\ell=\theta_\ell^{(0)}$ remains fixed for all $\ell\notin\Lambda$.

To evaluate a candidate subset $\Lambda$, we perform temporary training on $\widetilde{\mathcal{D}}_{\mathrm{syn}}^{\mathrm{tr}}$ under a perturbation regime matched to the downstream DP fine-tuning stage. Let
\begin{equation}
\mathcal{L}_{\mathrm{syn}}(\theta;\widetilde{\mathcal{D}}_{\mathrm{syn}}^{\mathrm{tr}})
=
\frac{1}{|\widetilde{\mathcal{D}}_{\mathrm{syn}}^{\mathrm{tr}}|}
\sum_{(\tilde{x},\tilde{y})\in \widetilde{\mathcal{D}}_{\mathrm{syn}}^{\mathrm{tr}}}
\ell(\theta;\tilde{x},\tilde{y})
\end{equation}
denote the synthetic training loss. For a given $\Lambda$, we first compute the gradient with respect to the selected layers,
\begin{equation}
g_\Lambda
=
\nabla_{\theta_\Lambda}
\mathcal{L}_{\mathrm{syn}}(\theta;\widetilde{\mathcal{D}}_{\mathrm{syn}}^{\mathrm{tr}}).
\end{equation}

Instead of applying the clean update direction $g_\Lambda$, we inject perturbation during this temporary training process to mimic the noisy optimization conditions in DP fine-tuning. Specifically, we use
\begin{equation}
\widehat{g}_\Lambda = g_\Lambda + \xi_\Lambda,
\end{equation}
where $\xi_\Lambda$ is a perturbation vector whose norm scale matches that of the downstream DP noise.

Rather than sampling $\xi_\Lambda$ as random Gaussian noise, we use a worst-case perturbation under the same norm budget. Concretely, we define
\begin{equation}
\xi_\Lambda^\star
\in
\arg\max_{\|\xi\|_2 \leq \rho_\Lambda}
\mathcal{L}_{\mathrm{syn}}
\bigl(\theta-\eta(g_\Lambda+\xi);\widetilde{\mathcal{D}}_{\mathrm{syn}}^{\mathrm{tr}}\bigr),
\end{equation}
where $\eta$ is the learning rate and $\rho_\Lambda$ is chosen to match the effective perturbation scale induced by the downstream DP mechanism. The temporary update is then given by
\begin{equation}
\theta_\Lambda'
=
\theta_\Lambda - \eta \bigl(g_\Lambda + \xi_\Lambda^\star\bigr).
\end{equation}

This worst-case perturbation yields a conservative estimate of the robustness of layer subset $\Lambda$: if a candidate subset continues to perform well even under the most adverse perturbation of the prescribed scale, it is more likely to remain effective under the stochastic noise encountered in the actual DP fine-tuning stage.

After temporary training, we evaluate the resulting model on the synthetic validation set. Let
\begin{equation}
\mathrm{Perf}(\Lambda)
=
\mathrm{Acc}\bigl(\theta'(\Lambda);\widetilde{\mathcal{D}}_{\mathrm{syn}}^{\mathrm{val}}\bigr)
\end{equation}
denote the validation performance of the temporarily updated model associated with $\Lambda$, where $\mathrm{Acc}(\cdot)$ can be replaced by any task-appropriate evaluation metric. We then select the final layer subset as
\begin{equation}
\Lambda^\star
=
\arg\max_{\Lambda \in \mathcal{Q}}
\mathrm{Perf}(\Lambda),
\end{equation}
where $\mathcal{Q}$ is the candidate family of layer subsets considered by the selector.
The selected subset $\Lambda^\star$ is subsequently used in the final DP fine-tuning stage on the private dataset. By performing selection on synthetic data and evaluating candidate subsets under perturbation-aware training, this stage avoids additional privacy cost while producing a selector that is better aligned with the noisy optimization dynamics of DP learning.

\para{Overall algorithm.}
We summarize the whole pipeline in \autoref{alg:DP-SFT} in \autoref{subsec:algorithm}.

\subsection{Theoretical Analyses}
\label{subsec:analysis}

In this section, we provide a theoretical analysis of \ourmethod{}. 
Our goal is to formalize why selective fine-tuning is beneficial under DP noise, and why the proposed worst-case perturbation criterion improves the robustness of layer selection. 
We state the main results informally here and defer the formal theorems and proofs to \autoref{app:theory}.

\para{Notation.}
Let $\Lambda \subseteq [L]$ denote a selected layer subset and let $d_{\Lambda}$ be the number of trainable parameters in $\Lambda$. 
We write $\theta_{\Lambda}$ for the trainable parameters and $\theta_{\bar{\Lambda}}$ for the frozen parameters. 
Let $F_{\mathrm{pri}}(\theta)$ be the empirical risk on the private dataset and $F_{\mathrm{syn}}(\theta)$ be the empirical risk on the synthetic training split. 
For a perturbation $\xi$ applied to the trainable parameters, we denote by
\begin{equation}
\mathcal{R}_{\mathrm{syn}}(\Lambda,\xi)
\end{equation}
the validation risk on the synthetic validation split after training only the layers in $\Lambda$ under perturbation $\xi$. 
Similarly, $\mathcal{R}_{\mathrm{pri}}(\Lambda,\xi)$ denotes the corresponding risk on the private task distribution.

We make the following standard assumptions for the analysis.

\noindent\textbf{Assumption 1 (Smoothness).}
For any layer subset $\Lambda$, the loss is $\beta$-smooth with respect to $\theta_{\Lambda}$:
\begin{equation}
F(\theta_{\Lambda}+\Delta)
\leq
F(\theta_{\Lambda})
+
\left\langle
\nabla_{\theta_{\Lambda}}F(\theta_{\Lambda}),\Delta
\right\rangle
+
\frac{\beta}{2}\|\Delta\|_2^2.
\end{equation}

\noindent\textbf{Assumption 2 (Bounded clipped gradients).}
After per-example clipping, the gradient used in DP fine-tuning satisfies
\begin{equation}
\|g_{\Lambda}\|_2 \leq C.
\end{equation}
The injected DP noise on the selected trainable subspace follows
\begin{equation}
z_{\Lambda}\sim \mathcal{N}(0,\sigma^2 C^2 I_{d_{\Lambda}}).
\end{equation}

\noindent\textbf{Assumption 3 (Synthetic-to-private transfer).}
For all candidate layer subsets $\Lambda\in\mathcal{Q}$ and all perturbations $\|\xi\|_2\leq \rho$, the synthetic validation risk uniformly approximates the private-task risk up to error $\tau$:
\begin{equation}
\left|
\mathcal{R}_{\mathrm{syn}}(\Lambda,\xi)
-
\mathcal{R}_{\mathrm{pri}}(\Lambda,\xi)
\right|
\leq
\tau.
\end{equation}

\noindent\textbf{Assumption 4 (Bounded validation loss).}
The validation loss is bounded by $B$:
\begin{equation}
0\leq \mathcal{R}_{\mathrm{pri}}(\Lambda,\xi)\leq B.
\end{equation}

\para{Privacy guarantee.}
The end-to-end privacy guarantee follows from composition and post-processing. 
The API-based candidate generation stage does not access $\mathcal{D}_{\mathrm{pri}}$ and therefore incurs no privacy cost. 
The DP synthetic-data construction stage satisfies $(\varepsilon_{\mathrm{syn}},\delta_{\mathrm{syn}})$-DP, and the final DP fine-tuning stage satisfies $(\varepsilon_{\mathrm{ft}},\delta_{\mathrm{ft}})$-DP. 
Since layer selection is performed only on the synthetic dataset, it is post-processing of a DP output and incurs no additional privacy cost. 
Thus, the full pipeline satisfies
\begin{equation}
(\varepsilon,\delta)
=
(
\varepsilon_{\mathrm{syn}}+\varepsilon_{\mathrm{ft}},
\delta_{\mathrm{syn}}+\delta_{\mathrm{ft}}
).
\end{equation}

\para{Theorem 1 (Informal: selective fine-tuning trades learning signal for reduced DP noise damage).}
Similar to smoothness-based analyses of private ERM~\cite{zhang2017efficient}, we decompose the effect of a noisy private update into a descent term and a perturbation-induced error term. In our setting, this decomposition reveals a DP-specific trade-off: restricting updates to a selected layer subset $\Lambda$ reduces the DP noise damage through the trainable dimension $d_\Lambda$, but it also reduces the available learning signal to the gradient projected onto $\Lambda$.

Under Assumptions 1 and 2, consider one DP update restricted to a layer subset $\Lambda$:
\begin{equation}
\theta_{\Lambda}^{+}
=
\theta_{\Lambda}
-
\eta(g_{\Lambda}+z_{\Lambda}),
\qquad
z_{\Lambda}\sim\mathcal{N}(0,\sigma^2 C^2 I_{d_\Lambda}).
\end{equation}
Then the expected one-step risk satisfies
\begin{equation}
\mathbb{E}
\left[
F_{\mathrm{pri}}(\theta^{+})
\right]
\leq
F_{\mathrm{pri}}(\theta)
-
\eta
\left\langle
\nabla_{\theta_{\Lambda}}F_{\mathrm{pri}}(\theta),
g_{\Lambda}
\right\rangle
+
\frac{\beta\eta^2}{2}\|g_{\Lambda}\|_2^2
+
\frac{\beta\eta^2}{2}d_{\Lambda}\sigma^2 C^2.
\end{equation}

Here, the first-order term
\begin{equation}
\left\langle
\nabla_{\theta_{\Lambda}}F_{\mathrm{pri}}(\theta),
g_{\Lambda}
\right\rangle
\end{equation}
captures the learning signal retained by the selected layers. If $\Lambda$ is too small or poorly chosen, this term may become weak because the update only follows the gradient projected onto the selected subspace. In contrast, the last term
\begin{equation}
d_{\Lambda}\sigma^2 C^2
\end{equation}
captures the DP noise damage, which decreases with the number of trainable parameters. Therefore, selective fine-tuning is beneficial only when the selected subset preserves sufficient learning signal while substantially reducing noise-induced damage.

\para{Theorem 2 (Informal: worst-case selection generalizes to unseen DP noise).}
Let the worst-case selection objective be
\begin{equation}
W_{\mathrm{syn}}(\Lambda)
=
\sup_{\|\xi\|_2\leq \rho}
\mathcal{R}_{\mathrm{syn}}(\Lambda,\xi),
\end{equation}
and let \ourmethod{} select
\begin{equation}
\Lambda^{\star}
=
\arg\min_{\Lambda\in\mathcal{Q}}
W_{\mathrm{syn}}(\Lambda).
\end{equation}
Assume the downstream DP perturbation $Z$ satisfies
\begin{equation}
\Pr(\|Z\|_2\leq \rho)\geq 1-\alpha.
\end{equation}
Then, under Assumptions 3 and 4, the private-task risk of the selected subset under unseen DP noise is bounded by
\begin{equation}
\mathbb{E}_{Z}
\left[
\mathcal{R}_{\mathrm{pri}}(\Lambda^{\star},Z)
\right]
\leq
\min_{\Lambda\in\mathcal{Q}}
\sup_{\|\xi\|_2\leq \rho}
\mathcal{R}_{\mathrm{pri}}(\Lambda,\xi)
+
2\tau
+
\alpha B.
\end{equation}

This theorem explains why worst-case perturbation is useful for DP layer selection. 
Random perturbation evaluates a candidate layer subset under one sampled noise realization, which may be benign and may not represent the perturbations encountered during final DP fine-tuning. 
In contrast, worst-case perturbation minimizes an upper bound over all perturbations within the target noise scale. 
Thus, if a layer subset performs well under the worst admissible perturbation, it is guaranteed to perform well under any random DP perturbation within the same radius, up to the synthetic-to-private transfer error and the tail probability of the DP noise.

\para{Implications.}
Together, the two theorems justify the design of \ourmethod{}. 
Theorem~1 shows that selective fine-tuning has an intrinsic advantage under DP because the noise damage grows with the trainable dimension. 
Theorem~2 shows that worst-case perturbation provides a robust selection criterion that better generalizes to unseen DP noise than a criterion based on one random perturbation. 
Therefore, \ourmethod{} improves the privacy--utility trade-off by jointly controlling where noisy updates are applied and whether the selected layers remain stable under DP perturbations.

\begin{table*}[!t]
\caption{Comparison of the performance of full DP fine-tuning, DP-LoRA, heuristic layer selection, and \ourmethod{} across various downstream tasks under different privacy budgets.}
\begin{tabular}{c|c||ccccc|ccccc}
\hline 
\multirow{2}{*}{\begin{tabular}{c}
Privacy\\
Budget
\end{tabular}} 
& \multirow{2}{*}{Method} & \multicolumn{5}{c|}{RoBERTa-Large} & \multicolumn{5}{c}{OPT-350M} \\
\cline{3-12}
 &  & MNLI & QQP & SST-2 & SST-5 & Average & MNLI & QQP & SST-2 & SST-5 & Average \\
\hline \hline
\multirow{5}{*}{$\varepsilon=1$} 
& Full Parameter        &76.30  &76.30  &93.01  &46.30  &72.98  &\textbf{40.80}  &63.80  &87.73  &47.30  &59.91  \\
& LoRA~\cite{yu2021differentially}                  &74.60  &76.30  &92.78  &41.30  &71.25  &34.50  &61.60  &84.86  &43.80  &56.19  \\
& Heuristic-based~\cite{liu2023differentially}       &76.10  &78.20  &93.46  &46.30  &73.52  &36.80  &61.20  &88.18  &47.30  &58.37  \\ \cline{2-12}
& \textbf{\ourmethod{}}   &\textbf{77.00}  &\textbf{78.60}  &\textbf{94.15}  &\textbf{48.30}  &\textbf{74.51}  &40.20  &\textbf{64.50}  &\textbf{88.76}  &\textbf{47.80}  &\textbf{60.32}  \\
& \textbf{\ourmethod{}+LoRA}   &74.50  &76.80  &93.00  &41.40  &71.43  &35.60  &61.70  &85.32  &43.10  &56.43  \\
\hline
\hline
\multirow{5}{*}{$\varepsilon=5$} 
& Full Parameter        &77.70  &77.20  &94.04  &50.40  &74.84  &47.50  &74.40  &89.45  &48.60  &64.99  \\
& LoRA~\cite{yu2021differentially}                  &79.80  &77.70  &93.00  &49.00  &74.88  &36.10  &63.10  &88.07  &46.20  &58.37  \\
& Heuristic-based~\cite{liu2023differentially}       &77.60  &78.50  &93.92  &49.30  &74.83  &42.70  &71.60  &89.79  &49.00  &63.27  \\ \cline{2-12}
& \textbf{\ourmethod{}}   &79.70  &79.10  &\textbf{94.15}  &\textbf{51.40}  &76.09  &\textbf{54.40}  &\textbf{74.60}  &\textbf{90.37}  &\textbf{50.00}  &\textbf{67.34}  \\
& \textbf{\ourmethod{}+LoRA}   &\textbf{80.50}  &\textbf{79.50}  &93.46  &51.20  &\textbf{76.17}  &38.80  &64.30  &88.53  &46.40  &59.51  \\
\hline
\hline
\multirow{5}{*}{$\varepsilon=10$} 
& Full Parameter        &78.20  &77.30  &94.15  &50.60  &75.06  &55.50  &74.80  &89.91  &48.80  &67.25  \\
& LoRA~\cite{yu2021differentially}                  &81.20  &78.70  &93.12  &50.90  &75.98  &38.70  &64.80  &88.88  &47.40  &59.95  \\
& Heuristic-based~\cite{liu2023differentially}       &77.90  &78.80  &93.69   &47.90  &74.57  &46.50  &74.50  &90.37  &48.60  &64.99  \\ \cline{2-12}
& \textbf{\ourmethod{}}   &80.30  &79.40  &\textbf{94.38}  &\textbf{51.60}  &76.42  &\textbf{59.00}  &\textbf{76.30}  &\textbf{90.48}  &\textbf{50.80}  &\textbf{69.15}  \\
& \textbf{\ourmethod{}+LoRA}   &\textbf{82.20}  &\textbf{79.60}  &93.81  &51.20  &\textbf{76.70}  &41.40  &66.80  &89.11  &48.20  &61.38  \\
\hline
\end{tabular}
\label{tab:main}
\end{table*}

\section{Evaluation}
\label{sec:eval}

\subsection{Experimental Setup}
\label{subsec:setup}

\para{Models and APIs.}
We evaluate \ourmethod{} on four widely used open-source foundation models from the OPT~\cite{zhang2022opt} and RoBERTa~\cite{liu2019roberta} families. These models are representative backbones for local NLP fine-tuning and are commonly used in both private and non-private adaptation studies. For synthetic-data generation, we use commercial LLM APIs only with high-level, non-sensitive prompts; the private data never leaves the local environment. More details are provided in \autoref{app:setup}.

\para{Datasets.}
We evaluate on four benchmark datasets from GLUE~\cite{GLUE} and related standard NLP benchmarks, covering sentiment analysis, natural language inference, and paraphrase detection. Specifically, we use SST-2 and SST-5~\cite{socher2013recursive}, MNLI~\cite{williams2018broad}, and QQP~\cite{cer2017semeval}. These datasets cover diverse language understanding tasks and are standard benchmarks for evaluating fine-tuning quality.

\para{Privacy settings.}
We consider three end-to-end privacy budgets, $\varepsilon \in \{1,5,10\}$, representing strict, moderate, and relatively relaxed privacy regimes. Following common practice, we set the failure probability to $\delta=10^{-5}$ for all experiments. Unless otherwise specified, we adopt a unified setting across all experiments, where the clipping threshold is set to $C=1$, the batch size is $64$, the number of training steps is $1000$, and the training set size is fixed to $1024$ records. In addition, we set the number of non-DP pre-fine-tuning steps to $300$ by default.

\para{Baselines.}
We compare \ourmethod{} against three representative baselines: full DP fine-tuning, which updates all parameters; DP-LoRA~\cite{hu2022lora}, a representative non-selection-based PEFT method; and heuristic layer selection~\cite{liu2023differentially}, a manually designed strategy that updates only attention layers. These baselines allow us to compare against full-parameter training, structured PEFT, and hand-crafted selective tuning rules.

\para{Configuration of \ourmethod{}.}
We evaluate two variants: \ourmethod{}, which directly updates the selected layers, and \ourmethod{}+LoRA, which applies LoRA only to the selected layers. By default, we allocate $\varepsilon_{\mathrm{syn}}=0.3$ to synthetic-data construction and use the remaining privacy budget for final DP fine-tuning. We select the top-$k$ layers according to the perturbation-aware validation performance on synthetic data, with $k=15$ by default. Additional implementation details, including prompt templates, API choices, synthetic-data splits, and layer-search details, are deferred to \autoref{app:setup}.

\subsection{Main Results}

\autoref{tab:main} reports the downstream accuracy under different models, tasks, and privacy budgets. Overall, \ourmethod{} achieves the best average accuracy in all six model-budget settings, demonstrating consistent improvement over full DP fine-tuning, DP-LoRA, and heuristic layer selection. When $\varepsilon=1$, \ourmethod{} improves the average accuracy over the strongest baseline from $73.52\%$ to $74.51\%$ on RoBERTa-Large, yielding a gain of $0.99$ percentage points. Under larger privacy budgets, the gains become clearer: at $\varepsilon=5$, \ourmethod{} improves the OPT-350M average accuracy from $64.99\%$ to $67.34\%$, a gain of $2.35$ percentage points; at $\varepsilon=10$, it improves the accuracy from $67.25\%$ to $69.15\%$, a gain of $1.90$ percentage points.

The only outlier is the $\varepsilon=1$ setting with OPT-350M, where Full Parameter fine-tuning slightly outperforms \ourmethod{} on MNLI ($40.80\%$ vs. $40.20\%$). This may be because MNLI requires broad cross-sentence reasoning, and under a very small model and strict privacy budget, restricting updates to a subset of layers can slightly limit the model's adaptation capacity. Nevertheless, \ourmethod{} still achieves the best average accuracy in this setting, indicating that the selected layers provide a better overall privacy--utility trade-off.

The results also show that \ourmethod{} is complementary to LoRA in some settings, but the benefit is model-dependent. On RoBERTa-Large, \ourmethod{}+LoRA achieves the best average accuracy under $\varepsilon=5$ and $\varepsilon=10$, suggesting that selecting where LoRA modules are applied can further improve non-selection-based PEFT. On OPT-350M, however, \ourmethod{} without LoRA consistently performs better than \ourmethod{}+LoRA. This indicates that, for some backbones, directly updating carefully selected layers is more effective than introducing low-rank adapters under DP noise.

We further observe that the improvement of \ourmethod{} is especially pronounced on challenging tasks. Compared with SST-2, SST-5 requires fine-grained sentiment classification and is more sensitive to noisy updates. This difference is reflected in the gains. Under $\varepsilon=1$ on RoBERTa-Large, \ourmethod{} improves SST-2 accuracy from $93.46\%$ to $94.15\%$, a gain of $0.69$ percentage points, while improving SST-5 accuracy from $46.30\%$ to $48.30\%$, a gain of $2.00$ percentage points. Similarly, under $\varepsilon=10$ on OPT-350M, \ourmethod{} improves SST-2 accuracy from $90.37\%$ to $90.48\%$, a gain of $0.11$ percentage points, while improving SST-5 accuracy from $48.80\%$ to $50.80\%$, a gain of $2.00$ percentage points. This is consistent with our motivation: fine-grained tasks are more vulnerable to DP-induced perturbations, and selecting layers that are both learnable and robust becomes particularly beneficial.

\begin{table}[!t]
\caption{Ablation Study on MNLI~\cite{williams2018broad} Dataset.}
\resizebox{.9\linewidth}{!}{
\begin{tabular}{l||ccc}
\hline 
Method & \multicolumn{1}{c}{$\varepsilon=1$} & \multicolumn{1}{c}{$\varepsilon=5$} & \multicolumn{1}{c}{$\varepsilon=10$} \\
\hline \hline
Full Parameter &76.30  &77.70  &78.20  \\
LoRA~\cite{yu2021differentially} &74.60  &79.80  &81.20  \\
Heuristic-based~\cite{liu2023differentially} &76.10  &77.60  &77.90  \\
Gradient-based~\cite{zhang2024gradient} &74.00  &77.60  &78.10  \\
Loss-based~\cite{he2023sensitivity} &74.00  &78.40  &79.20  \\
Performance-based (without noise) &67.00  &73.70  &75.20  \\
Performance-based (with noise) &67.10  &74.50  &77.50  \\ \hline \hline
\textbf{\ourmethod{}} &\textbf{77.00}  &79.70  &80.30 \\
\textbf{\ourmethod{}+LoRA} &74.50  &\textbf{80.50}  &\textbf{82.20} \\
\hline
\end{tabular}
}
\label{tab:ablation}
\end{table}

\subsection{Ablation Study}

In this section, we conduct ablation studies to evaluate the necessity of each design component in \ourmethod{} and to understand whether existing non-DP selection criteria can be directly adapted to DP fine-tuning. We fix the foundation model to RoBERTa-Large and report results on MNLI in \autoref{tab:ablation}.

We consider four additional selection baselines. 
(1) \emph{Gradient-based}~\cite{zhang2024gradient}, which ranks layers using gradient-based importance scores computed on the synthetic data. 
(2) \emph{Loss-based}~\cite{he2023sensitivity}, which ranks layers according to validation loss on the synthetic data. 
(3) \emph{Performance-based (without noise)}, which selects layers by validation accuracy after clean temporary training. 
(4) \emph{Performance-based (with noise)}, which adds clipping and random noise during temporary training and then selects layers by validation accuracy.

\begin{figure}[t]
\centering

\begin{subfigure}[t]{0.49\linewidth}
    \centering
    \includegraphics[width=\linewidth]{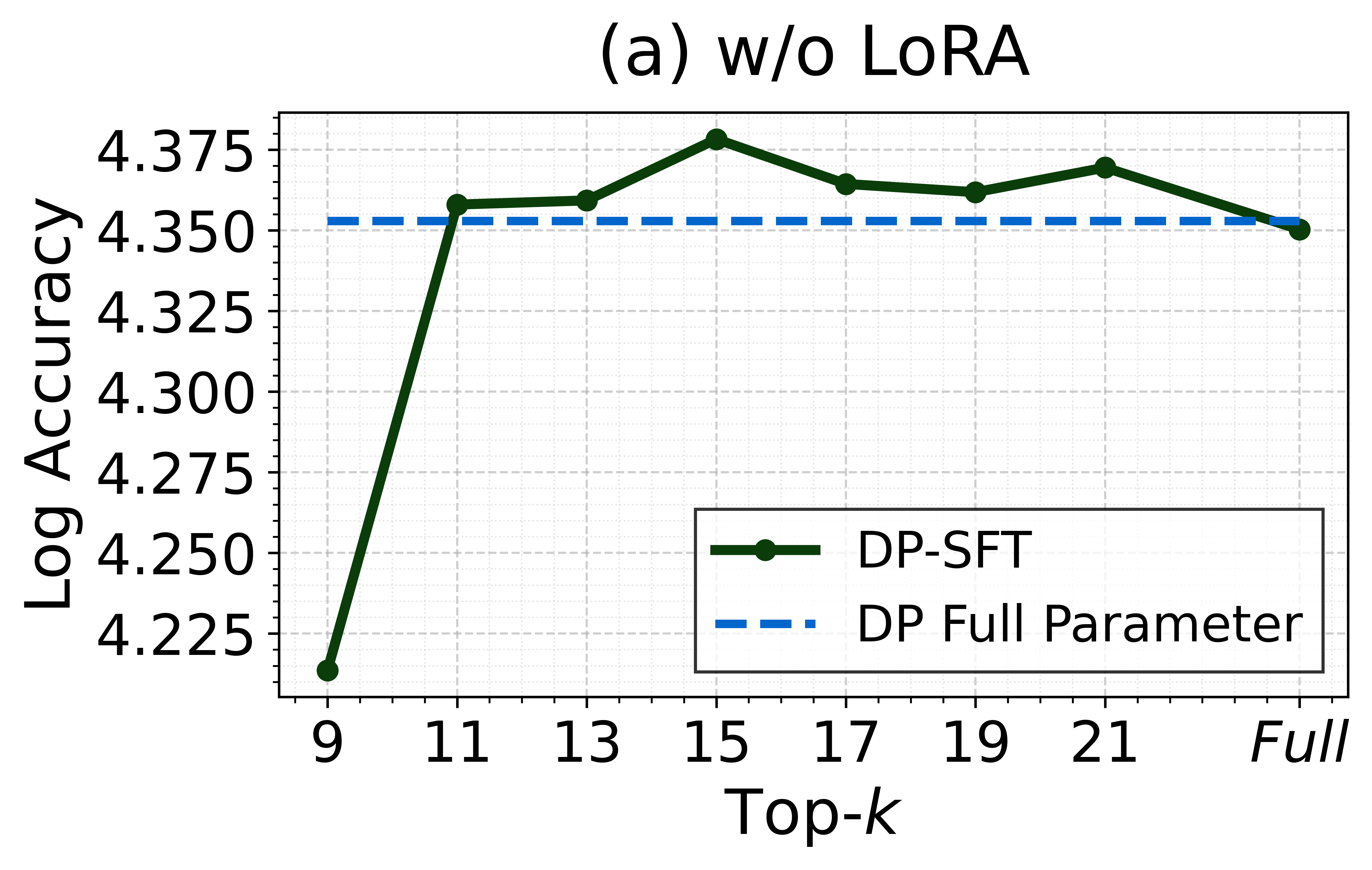}
\end{subfigure}
\hfill
\begin{subfigure}[t]{0.49\linewidth}
    \centering
    \includegraphics[width=\linewidth]{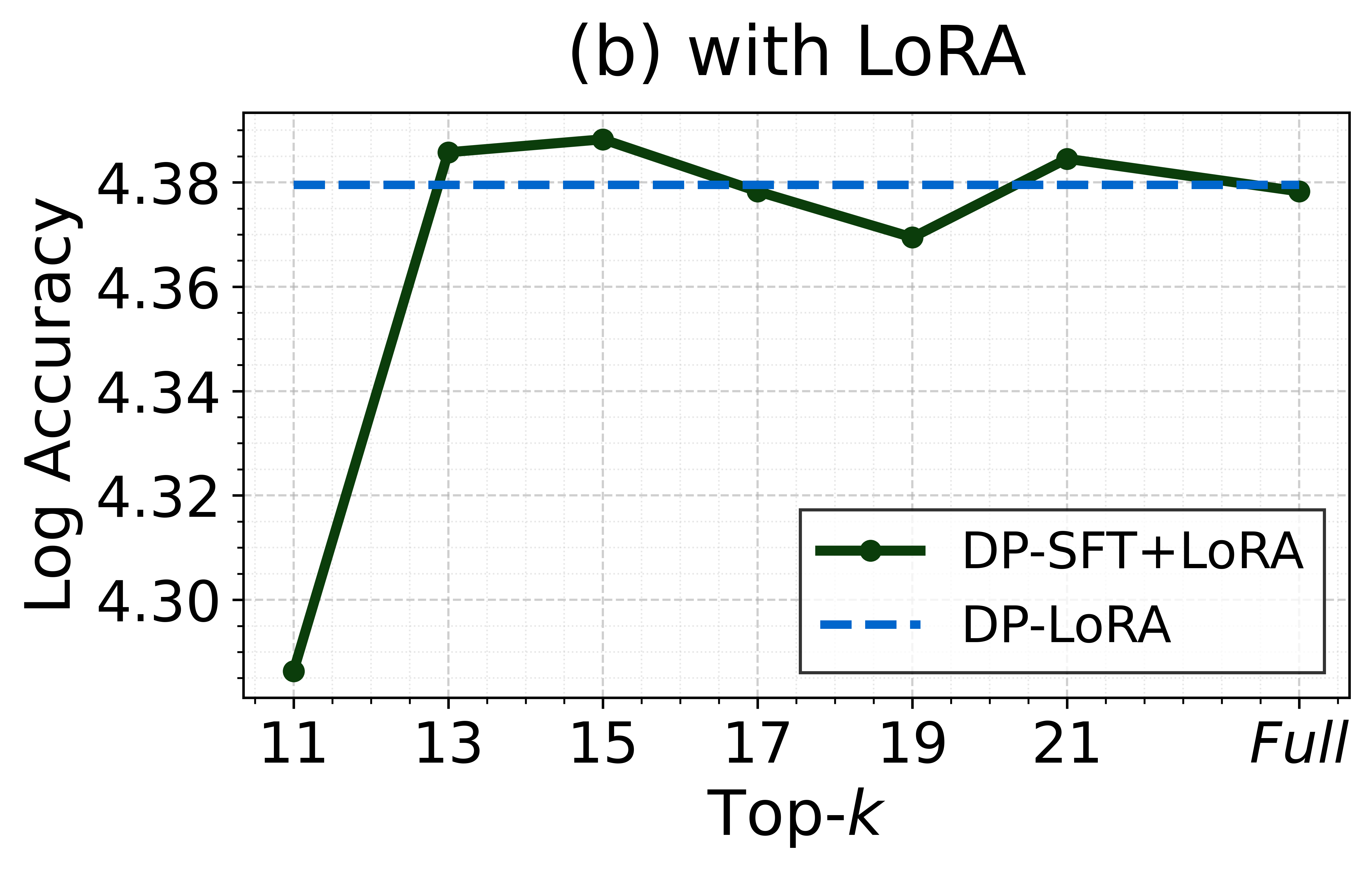}
\end{subfigure}

\caption{Impact of Top-$k$ on MNLI under $\varepsilon=5$ and RoBERTa-Large. Left: comparison between DP-SFT and DP Full Parameter. Right: comparison between DP-SFT+LoRA and DP-LoRA.}
\Description{}
\vspace{-5pt}
\label{fig:ablation_top-k}
\end{figure}

\autoref{tab:ablation} shows the results. First, directly adapting gradient- or loss-based selection~\cite{zhang2024gradient, he2023sensitivity} does not consistently outperform standard baselines. For example, under $\varepsilon=1$, both Gradient-based and Loss-based selection achieve $74.00\%$, below Full Parameter fine-tuning ($76.30\%$) and Heuristic-based selection ($76.10\%$). Under larger budgets, Loss-based selection improves, reaching $79.20\%$ at $\varepsilon=10$, but it still underperforms \ourmethod{}+LoRA ($82.20\%$). This suggests that gradient and loss signals alone are insufficient for reliable DP layer selection, since they do not directly capture how selected layers behave under noisy private optimization. In particular, loss-based selection is less effective than performance-based selection because validation loss and validation accuracy respond differently to DP-induced perturbations. 
Loss is sensitive to confidence calibration and probability margins over all classes, even when the predicted label does not change. 
As a result, it may favor layer subsets that improve probability estimates but are not necessarily stable under clipped and noisy updates. 
In contrast, accuracy directly measures whether the selected layers preserve correct task decisions after perturbation. 
Since the downstream objective is classification accuracy, performance-based selection provides a more task-aligned criterion for identifying layers that remain useful and robust under DP fine-tuning.

Second, performance-based selection without noise performs poorly, achieving only $67.00\%$, $73.70\%$, and $75.20\%$ under $\varepsilon=1,5,10$, respectively. Adding clipping and random noise improves the results to $67.10\%$, $74.50\%$, and $77.50\%$. This supports our claim that selection should reflect the perturbation regime of downstream DP fine-tuning rather than clean optimization. However, random-noise selection is still much worse than \ourmethod{}, indicating that a random perturbation can provide an unreliable estimate of robustness.

Finally, \ourmethod{} achieves the best result under the strictest privacy budget, improving over the strongest baseline from $76.30\%$ to $77.00\%$ at $\varepsilon=1$. When combined with LoRA, \ourmethod{}+LoRA achieves the best performance under larger budgets, reaching $80.50\%$ at $\varepsilon=5$ and $82.20\%$ at $\varepsilon=10$. These results validate the central design of \ourmethod{}: layer selection should be based not only on learnability, but also on robustness to adverse perturbations induced by DP training.

\subsection{Impact of Top-k}

In this section, we study the sensitivity of \ourmethod{} to the number of selected top-$k$ layers. We vary $k$ from a small subset to $\textit{Full}$, where $\textit{Full}$ means that all layers are updated. We consider two settings: \ourmethod{} without LoRA and \ourmethod{}+LoRA. We fix the foundation model to RoBERTa and the task to MNLI.

\autoref{fig:ablation_top-k} shows the downstream accuracy under different choices of $k$. For readability, we start the curves from $k=9$ without LoRA and $k=11$ with LoRA, since smaller values lead to extremely low accuracy. In both settings, \ourmethod{} achieves the best performance at an intermediate value, with the peak at $k=15$. Without LoRA, performance increases sharply from $k=9$ to $k=11$, reaches the best result at $k=15$, and then gradually declines as more layers are selected. With LoRA, the same pattern is observed: very small $k$ underperforms, while $k=13$ to $k=21$ remains competitive, with the best result again around $k=15$.

This trend reflects a natural trade-off in selective DP fine-tuning. When $k$ is too small, the trainable subspace is overly restricted, limiting the model's adaptation capacity and reducing the task-specific learning signal. When $k$ is too large, more layers are exposed to clipped and noisy DP updates, increasing the risk of damaging useful pre-trained representations and weakening the benefit of selection. The best performance is therefore achieved at an intermediate $k$, where the model retains sufficient adaptation flexibility while still reducing noise-induced damage.

Overall, these results show that \ourmethod{} is stable across a reasonable range of $k$ values and support our motivation that effective DP fine-tuning requires balancing learnability and robustness.

\subsection{Impact of Privacy Budget Allocation}

In this section, we study the sensitivity of \ourmethod{} to the privacy budget allocated to the synthetic-data construction stage under two settings: \ourmethod{} without LoRA and \ourmethod{}+LoRA. We fix the total privacy budget to $\varepsilon=5$, the foundation model to RoBERTa, and evaluate on MNLI and QQP. We vary the synthetic-data budget as $\varepsilon_{\mathrm{syn}} \in \{0.1, 0.3, 0.5, 1.0\}$, with the remaining budget allocated to the final DP fine-tuning stage.

\begin{table}[t]
\caption{Model utility under different privacy allocation.}
\centering
\resizebox{.95\linewidth}{!}{
\renewcommand{\arraystretch}{1.1}
\setlength{\tabcolsep}{12pt}
\begin{tabular}{c|c||cccc}
\hline
\multirow{2}{*}{Method} &
\multirow{2}{*}{Dataset} &
\multicolumn{4}{c}{$\varepsilon_{\mathrm{syn}}$} \\
\cline{3-6}
& &0.1 &0.3 &0.5 &1.0  \\
\hline
\hline
\multirow{2}{*}{\textbf{\ourmethod{}}}
& MNLI &79.00 &79.70 &\textbf{79.90} &78.80  \\
& QQP &78.30  &\textbf{79.10} &78.40 &78.60\ \\
\hline
\hline
\multirow{2}{*}{\textbf{\ourmethod{}+LoRA}}
& MNLI &79.50 &\textbf{80.50} &79.80 &80.10 \\
& QQP &79.00  &\textbf{79.50} &\textbf{79.50} &79.40 \ \\
\hline
\end{tabular}
}
\label{tab:ablation_budget}
\end{table}

\autoref{tab:ablation_budget} reports the results. Overall, \ourmethod{} is reasonably stable across different allocations, with the best performance typically achieved at an intermediate $\varepsilon_{\mathrm{syn}}$. For \ourmethod{}, the best result on MNLI is obtained at $\varepsilon_{\mathrm{syn}}=0.5$ ($79.90\%$), while the best result on QQP is obtained at $\varepsilon_{\mathrm{syn}}=0.3$ ($79.10\%$). For \ourmethod{}+LoRA, $\varepsilon_{\mathrm{syn}}=0.3$ gives the best MNLI result ($80.50\%$), and $\varepsilon_{\mathrm{syn}}\in\{0.3,0.5\}$ gives the best QQP result ($79.50\%$).
These results reflect a natural trade-off. When $\varepsilon_{\mathrm{syn}}$ is too small, the synthetic data may be less informative, weakening the quality of layer selection. For example, increasing $\varepsilon_{\mathrm{syn}}$ from $0.1$ to $0.3$ improves \ourmethod{}+LoRA from $79.50\%$ to $80.50\%$ on MNLI. On the other hand, allocating too much privacy budget to synthetic-data construction leaves less budget for final DP fine-tuning, which can hurt downstream utility. For example, \ourmethod{} on MNLI drops from $79.90\%$ at $\varepsilon_{\mathrm{syn}}=0.5$ to $78.80\%$ at $\varepsilon_{\mathrm{syn}}=1.0$.

Overall, \ourmethod{} is not overly sensitive to the exact allocation, but an intermediate $\varepsilon_{\mathrm{syn}}$ generally provides the best balance between reliable selection and effective final DP fine-tuning. We use $\varepsilon_{\mathrm{syn}}=0.3$ by default, as it performs strongly across settings and leaves sufficient privacy budget for strict regimes such as total $\varepsilon=1$.

\begin{table}[!t]
\caption{Impact of model size on MNLI dataset under $\varepsilon=5$.}
\resizebox{\linewidth}{!}{
\begin{tabular}{l||ccc}
\hline
Method & DistilRoberta & Roberta-Base & Roberta-Large \\
\hline \hline
Full Parameter &57.50  &71.00  &77.70  \\
LoRA~\cite{yu2021differentially} &53.70  &71.20  &79.80  \\
Heuristic-based~\cite{liu2023differentially} &59.20  &70.70  &77.60  \\ \hline \hline
\textbf{\ourmethod{}} &\textbf{59.80}  &\textbf{72.00}  &79.70  \\
\textbf{\ourmethod{}+LoRA} &55.90  &71.20  &\textbf{80.50}  \\
\hline
\end{tabular}
}
\label{tab:size}
\end{table}

\subsection{Impact of Model Size}

In this section, we investigate the impact of model size on \ourmethod{}. Specifically, we consider three RoBERTa-family models with different scales: DistilRoBERTa~\cite{sanh2019distilbert}, RoBERTa-Base~\cite{liu2019roberta}, and RoBERTa-Large~\cite{liu2019roberta}. We fix $\varepsilon=5$ and evaluate on MNLI.

\autoref{tab:size} reports the results. Overall, \ourmethod{} achieves the best performance on the smaller and medium-sized models, improving over the strongest baseline from $59.20\%$ to $59.80\%$ on DistilRoBERTa, a gain of $0.60$ percentage points, and from $71.20\%$ to $72.00\%$ on RoBERTa-Base, a gain of $0.80$ percentage points. On RoBERTa-Large, \ourmethod{}+LoRA achieves the best accuracy, improving over the strongest baseline from $79.80\%$ to $80.50\%$, a gain of $0.70$ percentage points. These results suggest that selective fine-tuning remains effective across model scales, but the best way to apply selection can depend on the backbone size.
For smaller models, directly updating the selected layers performs best. This is likely because smaller models have limited adaptation capacity, and adding LoRA may further constrain the update space. In contrast, for RoBERTa-Large, combining \ourmethod{} with LoRA performs best, suggesting that larger models can benefit from both selecting where updates are applied and constraining the form of those updates. This is consistent with our motivation: as the model becomes larger, controlling noisy DP updates becomes more important, and combining layer selection with structured PEFT can better preserve useful pre-trained representations while still enabling task adaptation.

Overall, these results show that \ourmethod{} is robust across different model sizes and can be used either as a standalone strategy or together with LoRA, depending on the model scale.

\begin{table}[!t]
\caption{Impact of remote API under $\varepsilon=5$.}
\resizebox{\linewidth}{!}{
\begin{tabular}{lc||ccc}
\hline
Fine-tuning & Pre-fine-tuning & Claude Haiku & GPT-4o-mini & GPT-4o \\
\hline\hline
\multirow{2}{*}{Full Parameter} & None &77.70  &77.70  &77.70  \\
 & Full Parameter &79.90  &77.80  &78.70  \\
\hline
\multirow{3}{*}{LoRA~\cite{yu2021differentially}} & None &79.80  &79.80  &79.80  \\
 & Full Parameter &81.80  &80.50  &81.90  \\
 & LoRA &81.30 &80.10 &80.40  \\
\hline
\multirow{3}{*}{Heuristic-based~\cite{liu2023differentially}} & None &77.60  &77.60  &77.60  \\
 & Full Parameter &80.80  &76.30  &80.10  \\ 
 & Heuristic-based &79.70  &75.80  &79.10  \\
 \hline \hline
\multirow{3}{*}{\textbf{\ourmethod{}}} & None &78.10  &79.00  &79.70  \\
 & Full Parameter &79.80  &79.90  &81.00  \\
 & \ourmethod{} &78.50  &78.40  &80.60   \\
\hline
 \multirow{3}{*}{\textbf{\ourmethod{}+LoRA}} & None &79.50  &80.30  &80.50  \\
 & Full Parameter &\textbf{82.00}  &80.40  &\textbf{82.40}  \\
 & \ourmethod{}+LoRA &80.60  &\textbf{81.00}  &80.90   \\
\hline
\end{tabular}
}
\label{tab:API}
\end{table}

\subsection{Impact of Remote API}

Since \ourmethod{} leverages knowledge from a remote API during the synthetic-data construction stage, one may question whether the comparison with baselines is entirely fair. To address this concern, we strengthen the baselines by allowing them to also use remote API access. Specifically, for each baseline, we construct API-enhanced variants that first use the remote API to generate a differentially private synthetic dataset, then perform an additional pre-fine-tuning stage on this synthetic data, and finally perform the original DP fine-tuning stage on the private dataset. We consider two pre-fine-tuning strategies: full-parameter pre-fine-tuning, which updates all model parameters on the synthetic data, and method-aligned pre-fine-tuning, which uses the same adaptation strategy as the final DP fine-tuning method, such as LoRA for DP-LoRA and heuristic layer selection for heuristic-based DP fine-tuning. We tune the hyperparameters of each variant independently for a fair comparison.

We also study the impact of different remote APIs, including Claude Haiku, GPT-4o-mini, and GPT-4o, with additional API and implementation details provided in \autoref{app:setup}. In this section, we fix the privacy budget to $\varepsilon=5$, the foundation model to RoBERTa, and the task to MNLI.

\autoref{tab:API} reports the results. Overall, API-enhanced baselines do not consistently outperform their original counterparts. In some cases, pre-fine-tuning on API-generated synthetic data improves performance, while in others the gain is limited or even disappears. This suggests that directly using synthetic data for pre-fine-tuning imposes a much stronger requirement on data quality than using it for selection. The synthetic data must be accurate, task-aligned, and sufficiently informative to provide a useful optimization signal for model adaptation.
This also creates a difficult privacy-budget trade-off. Allocating more privacy budget to synthetic-data construction can improve the quality of the synthetic data, but leaves less budget for the final DP fine-tuning stage. Conversely, allocating less budget preserves more privacy resources for final fine-tuning, but the synthetic data may become too noisy or mismatched to provide meaningful benefit during pre-fine-tuning. Across the tested APIs, we observe broadly similar trends: changing the remote API does not fundamentally change whether API-enhanced pre-fine-tuning is helpful. This indicates that the main bottleneck is not the specific remote API used, but the difficulty of converting privacy-limited synthetic data into an effective pre-fine-tuning signal.

In contrast, \ourmethod{} benefits more reliably from remote API knowledge because it uses synthetic data only for layer selection rather than direct model adaptation. This imposes a weaker requirement on synthetic data quality: the data only needs to preserve enough task structure to distinguish better layer subsets from worse ones. Importantly, even after strengthening all baselines with remote API access and independently tuned pre-fine-tuning variants, \ourmethod{} still achieves the best performance across all tested APIs. This further supports that the gain of \ourmethod{} does not simply come from access to stronger generators, but from using synthetic data in a DP-specific and selection-oriented way.
\section{Conclusion}
We studied selective fine-tuning for differentially private LLM adaptation, showing that DP fine-tuning should consider not only the \emph{form} of updates, as in existing PEFT methods, but also \emph{where} noisy updates are applied. We identified key challenges in directly adapting non-private selection methods, including privacy cost, unstable noisy estimates, and mismatch with DP optimization. To address them, we proposed \ourmethod{}, a DP-specific framework combining lightweight DP synthetic data construction, layer-level selection, and worst-case perturbation. Experiments show that \ourmethod{} consistently improves the privacy--utility trade-off over existing baselines under the same privacy guarantees.


\appendix

\section*{Ethical Considerations}
We conducted this research in accordance with established ethical guidelines and best practices. All experiments were performed in a controlled local environment using publicly available datasets and open-source models. The study relies exclusively on public benchmark data and does not involve human participants, user studies, or the collection of private or personally identifiable information. The proposed method is designed to improve privacy-preserving model adaptation under formal differential privacy guarantees.

\section*{Open Science}
The implementation of our evaluation and the proposed \ourmethod{} is available at:
\url{https://anonymous.4open.science/r/DPSelectiveFT/README.md}.
The code is released for research and evaluation purposes only, with the goal of facilitating reproducibility and further study of differentially private LLM fine-tuning.

\section*{Generative AI Usage}
ChatGPT was used for minor grammar correction, language polishing, and writing assistance. All technical ideas, experimental designs, analyses, and conclusions were developed and verified by the authors.

\newpage

\section{Algorithms}
\label{subsec:algorithm}
\autoref{alg:DP-SFT} summarizes \ourmethod{}. The algorithm first constructs a labeled DP synthetic dataset from API-generated candidates, and splits it into training and validation subsets. It then evaluates each candidate layer subset under worst-case perturbation on the synthetic data and selects the subset with the best validation performance. Finally, it performs DP-AdamW fine-tuning on the private dataset while updating only the selected layers.

\begin{algorithm}[ht]
\caption{Differentially Private Selective Fine-tuning}
\label{alg:DP-SFT}
\begin{algorithmic}[1]
\REQUIRE Private dataset $\mathcal{D}_{\mathrm{pri}}$; foundation model $f_{\theta^{(0)}}$; remote API $\mathcal{A}$; task-level prompt template $\pi$; public encoder $\phi(\cdot)$; candidate layer family $\mathcal{Q}$; privacy budgets $(\varepsilon_{\mathrm{syn}},\delta_{\mathrm{syn}})$ and $(\varepsilon_{\mathrm{ft}},\delta_{\mathrm{ft}})$
\ENSURE Final DP fine-tuned model $\theta^\star$

\STATE \textbf{Stage 1: Synthetic data generation}
\STATE Generate unlabeled candidate pool $\mathcal{C}$ from $\mathcal{A}$ using $\pi$
\STATE Construct labeled synthetic dataset
\[
\widetilde{\mathcal{D}}_{\mathrm{syn}}
\leftarrow
\mathsf{DP\text{-}Select\text{-}Label}
\bigl(\mathcal{D}_{\mathrm{pri}}, \mathcal{C}, \phi, \varepsilon_{\mathrm{syn}}, \delta_{\mathrm{syn}}\bigr)
\]
\STATE Split
\[
\widetilde{\mathcal{D}}_{\mathrm{syn}}
=
\widetilde{\mathcal{D}}_{\mathrm{syn}}^{\mathrm{tr}}
\cup
\widetilde{\mathcal{D}}_{\mathrm{syn}}^{\mathrm{val}}
\]

\STATE \textbf{Stage 2: Parameter selection}
\STATE Decompose $\theta^{(0)}$ layer-wise as $\theta^{(0)}=\{\theta_1^{(0)},\dots,\theta_L^{(0)}\}$
\FOR{each $\Lambda \in \mathcal{Q}$}
    \STATE Initialize temporary model $\theta \leftarrow \theta^{(0)}$ with only $\theta_\Lambda$ trainable
    \STATE Compute synthetic training loss
    \[
    \mathcal{L}_{\mathrm{syn}}(\theta;\widetilde{\mathcal{D}}_{\mathrm{syn}}^{\mathrm{tr}})
    =
    \frac{1}{|\widetilde{\mathcal{D}}_{\mathrm{syn}}^{\mathrm{tr}}|}
    \sum_{(\tilde{x},\tilde{y})\in\widetilde{\mathcal{D}}_{\mathrm{syn}}^{\mathrm{tr}}}
    \ell(\theta;\tilde{x},\tilde{y})
    \]
    \STATE Compute gradient on selected layers
    \[
    g_\Lambda
    =
    \nabla_{\theta_\Lambda}
    \mathcal{L}_{\mathrm{syn}}(\theta;\widetilde{\mathcal{D}}_{\mathrm{syn}}^{\mathrm{tr}})
    \]
    \STATE Compute worst-case perturbation
    \[
    \xi_\Lambda^\star
    \in
    \arg\max_{\|\xi\|_2\le\rho_\Lambda}
    \mathcal{L}_{\mathrm{syn}}
    \bigl(
    \theta-\eta(g_\Lambda+\xi);
    \widetilde{\mathcal{D}}_{\mathrm{syn}}^{\mathrm{tr}}
    \bigr)
    \]
    \STATE Score candidate subset on validation split
    \[
    \mathrm{Perf}(\Lambda)
    =
    \mathrm{Acc}\bigl(\theta'(\Lambda);\widetilde{\mathcal{D}}_{\mathrm{syn}}^{\mathrm{val}}\bigr)
    \]
\ENDFOR
\STATE Select layers
\[
\Lambda^\star
=
\arg\max_{\Lambda\in\mathcal{Q}}
\mathrm{Perf}(\Lambda)
\]

\STATE \textbf{Stage 3: DP fine-tuning}
\STATE Initialize $\theta \leftarrow \theta^{(0)}$ with only $\theta_{\Lambda^\star}$ trainable
\FOR{each iteration $t=1,\dots,T$}
    \STATE Sample mini-batch $\mathcal{B}_t \subseteq \mathcal{D}_{\mathrm{pri}}$
    \STATE Perform DP-AdamW update on $\theta_{\Lambda^\star}$ using $\mathcal{B}_t$
\ENDFOR
\STATE \textbf{return} $\theta^\star$
\end{algorithmic}
\end{algorithm}

\newpage

\section{Additional Experimental Setup}
\label{app:setup}

\para{Models and remote APIs.}
We evaluate \ourmethod{} on four widely used foundation models from the OPT~\cite{zhang2022opt} and RoBERTa~\cite{liu2019roberta} families. These models are representative open-source backbones for local fine-tuning and are commonly used in prior work on private and non-private NLP adaptation. For remote API access in the synthetic-data generation stage, we use OpenAI GPT-4o-mini, OpenAI GPT-4o\footnote{\url{https://openai.com/index/hello-gpt-4o/}}, and Claude Haiku 4.5\footnote{\url{https://www.anthropic.com/news/claude-haiku-4-5}}. These APIs are widely deployed, provide strong general-purpose text generation capability, and are inexpensive enough for lightweight candidate data generation at the scale considered in our experiments. This setup matches our threat model: the defender can use popular commercial APIs for non-sensitive, high-level prompt-based generation, while all operations involving private data are performed locally.

\para{Dataset details.}
We evaluate on four benchmark datasets from GLUE~\cite{GLUE} and related standard NLP benchmarks. SST-2 and SST-5~\cite{socher2013recursive} are used for sentiment analysis: SST-2 is binary sentiment classification, while SST-5 requires fine-grained sentiment classification. MNLI~\cite{williams2018broad} is used for natural language inference, where the goal is to determine whether a hypothesis is entailed by, neutral with respect to, or contradicted by a premise. QQP~\cite{cer2017semeval} is used for paraphrase detection, where the task is to determine whether two questions are semantically equivalent.

\para{Baseline details.}
Full DP fine-tuning updates all model parameters under DP training. DP-LoRA~\cite{hu2022lora} constrains the form of updates by introducing low-rank adaptation modules and serves as a representative non-selection-based PEFT baseline. Heuristic layer selection~\cite{liu2023differentially} is a manually designed, non-adaptive strategy that updates only attention layers. This heuristic is motivated by prior observations that modifying attention modules can substantially affect LLM behavior~\cite{shi2023toast, hu2022lora}. It allows us to assess the benefit of automatic, perturbation-aware layer selection over hand-crafted selective tuning rules.

\para{Configuration of \ourmethod{}.}
We consider two variants of \ourmethod{}. The first directly updates the selected layers and evaluates the benefit of selective fine-tuning alone. The second, denoted as \ourmethod{}+LoRA, combines layer selection with LoRA and tests whether \ourmethod{} can complement non-selection-based PEFT methods by determining where LoRA modules should be applied.

The prompt templates $\pi$ used for different datasets are summarized in \autoref{subsec:prompts}. These prompts are intentionally high-level and task-oriented, and contain no private examples, identifiers, or dataset-specific records.

By default, we allocate $\varepsilon_{\mathrm{syn}}=0.3$ to the synthetic-data construction stage. This empirical choice keeps the privacy overhead of selection small while leaving most of the privacy budget for final DP fine-tuning. After constructing the synthetic dataset, we randomly split it into synthetic training and validation subsets with a ratio of $7{:}3$.

For the candidate layer family $\mathcal{Q}$, we adopt an efficiency-oriented approximation. Instead of evaluating arbitrary layer subsets, we evaluate each layer independently during selection. Specifically, for each layer $\ell$, we temporarily unfreeze only that layer while keeping all other layers frozen, and then compute its validation performance under our perturbation-aware selection procedure. We rank all layers by validation performance and select the top-$k$ layers for final DP fine-tuning. Unless otherwise specified, we set $k=15$.

Unless otherwise specified, we use the same failure probability for synthetic-data construction and final fine-tuning, i.e., $\delta_{\mathrm{syn}}=\delta_{\mathrm{ft}}$. The privacy budget consumed by synthetic-data construction is included in the end-to-end privacy budget, and the remaining budget is used for final DP fine-tuning.

\para{Configuration of generation.}
We construct synthetic datasets via a variation-based expansion process initialized from a small set of seed samples. For each seed, we generate $\mu=3$ variations, resulting in $L=\mu+1=4$ candidates per instance. The number of seed samples $N_{\text{seed}}$ is set to $800$ for MNLI and $600$ for QQP, SST-2, and SST-5, yielding a total of $L \times N_{\text{seed}}$ synthetic samples. To ensure a consistent data scale and a fair comparison, we subsample the synthetic data to match the size of the downstream training set for both layer selection and model pre-fine-tuning.

We encode all candidate samples using \texttt{sentence-t5-xl} to obtain semantic representations for subsequent data generation and selection. The text generation is performed using a primary generator deployed on the Azure platform\footnote{\url{https://portal.azure.com/}}, with temperature of $1.2$ to encourage diversity. The maximum generation length is set to $30$ for MNLI, $25$ for QQP, and $10$ for SST-2 and SST-5. Variations are constructed using task-specific templates (please refer to ~\autoref{subsec:prompts}), with a perturbation probability $p$ controlling the degree of modification. We set $p=0.5$ for MNLI and SST-5, and $p=0.6$ for QQP and SST-2. To further regulate the diversity of generated samples, we apply a word-level variation scale of $30$ for MNLI and $20$ for the other datasets, together with a maximum token-level scaling factor of $2$.

To ensure differential privacy, we calibrate the noise multiplier following~\cite{xie2024differentially} and inject noise during the data generation process. Since high-fidelity synthetic data is not strictly required in our setting, we limit the number of iterative refinement epochs to $3$ to reduce the overall magnitude of injected noise. Moreover, due to the rewriting process, repeated iterations may introduce misalignment between generated sentences and their labels. Restricting the number of epochs helps mitigate such accumulation of semantic drift and label inconsistency. All candidate samples are generated independently and filtered in subsequent stages.

\section{Prompts}
\label{subsec:prompts}

This section summarizes the prompt templates used for synthetic data generation and paraphrasing. All prompts are intentionally high-level and task-oriented. They describe only the dataset format and task semantics, without including private examples, identifiers, or dataset-specific records.

\begin{figure}[!ht]
\centering
\begin{tcblisting}{
    enhanced,
    width=0.95\linewidth,
    colback=gray!5,
    colframe=black!75,
    coltitle=white,
    colbacktitle=black!75,
    fonttitle=\bfseries\large,
    title={Prompt Template for SST-2 Synthetic Data Generation},
    boxrule=1pt,
    arc=2.5mm,
    left=3mm,
    right=3mm,
    top=2mm,
    bottom=2mm,
    listing only,
    listing options={style=promptstyle}
}
Generate a short movie review snippet in the style of the SST-2 dataset.
Snippets are concise opinionated fragments (5-25 words) taken from real movie reviews.
The sentiment should be either positive or negative, and clearly expressed through natural language.
\end{tcblisting}
\caption{Prompt template for SST-2 synthetic data generation.}
\label{fig:prompt-sst2}
\Description{}
\end{figure}

\begin{figure}[!ht]
\centering
\begin{tcblisting}{
    enhanced,
    width=0.95\linewidth,
    colback=gray!5,
    colframe=black!75,
    coltitle=white,
    colbacktitle=black!75,
    fonttitle=\bfseries\large,
    title={Prompt Template for SST-2 Paraphrasing},
    boxrule=1pt,
    arc=2.5mm,
    left=3mm,
    right=3mm,
    top=2mm,
    bottom=2mm,
    listing only,
    listing options={style=promptstyle}
}
Paraphrase the movie review snippet while strictly preserving the original sentiment.
Output ONLY the paraphrased snippet -- no quotes, no labels, no explanations.
\end{tcblisting}
\caption{Prompt template for SST-2 paraphrasing.}
\label{fig:prompt-sst2-rephrase}
\Description{}
\end{figure}

\begin{figure}[!ht]
\centering
\begin{tcblisting}{
    enhanced,
    width=0.95\linewidth,
    colback=gray!5,
    colframe=black!75,
    coltitle=white,
    colbacktitle=black!75,
    fonttitle=\bfseries\large,
    title={Prompt Template for SST-5 Synthetic Data Generation},
    boxrule=1pt,
    arc=2.5mm,
    left=3mm,
    right=3mm,
    top=2mm,
    bottom=2mm,
    listing only,
    listing options={style=promptstyle}
}
Generate a short movie review snippet in the style of the SST-5 dataset.
Snippets are concise opinionated fragments (5-25 words) taken from real movie reviews.
The sentiment must be one of five categories: very negative, negative, neutral, positive, or very positive.
\end{tcblisting}
\caption{Prompt template for SST-5 synthetic data generation.}
\label{fig:prompt-sst5}
\Description{}
\end{figure}

\begin{figure}[!ht]
\centering
\begin{tcblisting}{
    enhanced,
    width=0.95\linewidth,
    colback=gray!5,
    colframe=black!75,
    coltitle=white,
    colbacktitle=black!75,
    fonttitle=\bfseries\large,
    title={Prompt Template for SST-5 Paraphrasing},
    boxrule=1pt,
    arc=2.5mm,
    left=3mm,
    right=3mm,
    top=2mm,
    bottom=2mm,
    listing only,
    listing options={style=promptstyle}
}
Paraphrase the movie review snippet while strictly preserving both the original sentiment and its intensity.
Output ONLY the paraphrased snippet -- no quotes, no labels, no explanations.
\end{tcblisting}
\caption{Prompt template for SST-5 paraphrasing.}
\label{fig:prompt-sst5-rephrase}
\Description{}
\end{figure}

\begin{figure}[!ht]
\centering
\begin{tcblisting}{
    enhanced,
    width=0.95\linewidth,
    colback=gray!5,
    colframe=black!75,
    coltitle=white,
    colbacktitle=black!75,
    fonttitle=\bfseries\large,
    title={Prompt Template for MNLI Synthetic Data Generation},
    boxrule=1pt,
    arc=2.5mm,
    left=3mm,
    right=3mm,
    top=2mm,
    bottom=2mm,
    listing only,
    listing options={style=promptstyle}
}
Generate a premise-hypothesis sentence pair for the MultiNLI task.
entailment: the hypothesis is definitely true given the premise.
neutral: the hypothesis may or may not be true given the premise.
contradiction: the hypothesis is definitely false given the premise.

Output ONLY the premise and hypothesis separated by a SINGLE TAB CHARACTER.
\end{tcblisting}
\caption{Prompt template for MNLI synthetic data generation.}
\label{fig:prompt-mnli}
\Description{}
\end{figure}

\begin{figure}[!ht]
\centering
\begin{tcblisting}{
    enhanced,
    width=0.95\linewidth,
    colback=gray!5,
    colframe=black!75,
    coltitle=white,
    colbacktitle=black!75,
    fonttitle=\bfseries\large,
    title={Prompt Template for MNLI Paraphrasing},
    boxrule=1pt,
    arc=2.5mm,
    left=3mm,
    right=3mm,
    top=2mm,
    bottom=2mm,
    listing only,
    listing options={style=promptstyle}
}
Paraphrase the tab-separated premise-hypothesis pair while preserving the logical relationship.
Output ONLY:
premise    hypothesis (one tab separator, no other text or newlines).
\end{tcblisting}
\caption{Prompt template for MNLI paraphrasing.}
\label{fig:prompt-mnli-rephrase}
\Description{}
\end{figure}

\begin{figure}[!ht]
\centering
\begin{tcblisting}{
    enhanced,
    width=0.95\linewidth,
    colback=gray!5,
    colframe=black!75,
    coltitle=white,
    colbacktitle=black!75,
    fonttitle=\bfseries\large,
    title={Prompt Template for QQP Synthetic Data Generation},
    boxrule=1pt,
    arc=2.5mm,
    left=3mm,
    right=3mm,
    top=2mm,
    bottom=2mm,
    listing only,
    listing options={style=promptstyle}
}
Generate a question pair for the Quora Question Pairs (QQP) task.
duplicate: both questions ask the same thing in different words.
not_duplicate: the questions ask clearly different things.

Output ONLY the two questions separated by a SINGLE TAB CHARACTER.
No labels, no newlines, no extra text.
\end{tcblisting}
\caption{Prompt template for QQP synthetic data generation.}
\label{fig:prompt-qqp}
\Description{}
\end{figure}

\begin{figure}[!ht]
\centering
\begin{tcblisting}{
    enhanced,
    width=0.95\linewidth,
    colback=gray!5,
    colframe=black!75,
    coltitle=white,
    colbacktitle=black!75,
    fonttitle=\bfseries\large,
    title={Prompt Template for QQP Paraphrasing},
    boxrule=1pt,
    arc=2.5mm,
    left=3mm,
    right=3mm,
    top=2mm,
    bottom=2mm,
    listing only,
    listing options={style=promptstyle}
}
Paraphrase the tab-separated question pair while preserving whether they are duplicate or not.
Output ONLY:
question1    question2 (one tab separator, no other text or newlines).
\end{tcblisting}
\caption{Prompt template for QQP paraphrasing.}
\label{fig:prompt-qqp-rephrase}
\Description{}
\end{figure}


\section{Discussion}
\label{sec:discussion}

\para{Why does the synthetic-data-based selection work?}
A natural question is why \ourmethod{} can make reliable selection decisions using synthetic data, even though the synthetic data is only a rough approximation of the private task distribution. We believe the key reason is that parameter selection imposes a much lower requirement on data quality than final model training. The goal of the synthetic dataset in \ourmethod{} is not to fully substitute for the private dataset or to support high-quality standalone model adaptation. Instead, it only needs to preserve enough coarse task structure to distinguish relatively better layer subsets from worse ones. In other words, selection mainly depends on comparative signals, rather than on perfectly accurate modeling of the target distribution. This makes the process substantially more tolerant to moderate distribution mismatch or label noise in the synthetic data. Moreover, our selection is performed at the layer level rather than the parameter level, which further reduces the granularity of the decision and improves robustness to synthetic-data imperfections. Taken together, these properties make it possible for a lightweight DP synthetic dataset to serve as an effective surrogate for privacy-free selection.

\para{Limitations and future work.}
Our work has several limitations that suggest directions for future research. First, \ourmethod{} requires choosing the number of selected top-$k$ layers, which introduces an additional hyperparameter. Although our experiments show that the method is reasonably stable across a range of $k$ values, automatically determining an appropriate selection budget would make the framework more practical and easier to deploy. Future work may explore adaptive or data-dependent strategies for selecting $k$ under privacy constraints. Second, our approach assumes that the remote API can generate task-relevant candidate samples from high-level prompts. When the private task is highly specialized, domain-specific, or contains substantial jargon, this assumption may become weaker, and the quality of the candidate pool may degrade accordingly. In such cases, the effectiveness of the synthetic-data-based selection stage may also decrease. One possible direction is to replace the generic API-based candidate generator with a differentially private task-specific generator trained locally on the private data. Such a generator could better capture the structure of specialized domains while preserving privacy. We leave this extension to future work.

\section{Related Work}
\label{sec:related}
\textbf{Differential privacy (DP)}~\cite{dwork2008differential, dwork2014algorithmic} provides a rigorous mathematical framework for quantifying privacy guarantees, typically parameterized by a privacy budget that captures the trade-off between protection and utility. It has been widely applied to diverse data analysis tasks, including synthetic dataset generation~\cite{yuan2023privgraph, zhang2021privsyn, wang2023privtrace, du2023ldptrace}, marginal release~\cite{zhang2018calm}, range queries~\cite{du2021ahead}, and streaming data processing~\cite{wang2021continuous}. For machine learning, a seminal advancement is DP-SGD~\cite{abadi2016deep}, which clips per-example gradients and injects calibrated noise to provide formal privacy guarantees. Subsequent works have sought to mitigate DP's utility overhead through new training algorithms~\cite{papernot2016semi, yu2021not} or context-specific relaxations of the DP definition~\cite{machanavajjhala2008privacy, ghazi2021deep}.

Orthogonal to our work, recent studies improve robustness to DP noise without explicitly reducing the trainable model size. In particular, several works~\cite{park2023differentially, shi2023make, wang2025sharpness, wang2024dpadapter} incorporate sharpness-aware minimization (SAM) or adversarial training techniques into DP learning, either during pre-training or private fine-tuning, aiming to make model parameters less sensitive to the perturbations introduced by DP optimization. In contrast, \ourmethod{} improves robustness to DP noise by restricting where noisy updates are applied during private fine-tuning. These two directions are complementary: robustness-oriented training reduces parameter sensitivity to noise, while \ourmethod{} reduces the parameter subspace exposed to noisy updates.
However, existing robustness-oriented DP methods have mainly been evaluated on image classification tasks. It remains unclear whether they directly generalize to large language models, where the model scale, optimization dynamics, and fine-tuning objectives differ substantially from vision models. We leave a systematic study of the compatibility between \ourmethod{} and robustness-oriented pre-training methods for future work.

\textbf{DP synthetic text generation} aims to produce text data that preserves useful distributional properties of the original data while limiting individual-level leakage~\cite{mattern2022differentially}. Existing studies have explored training or fine-tuning language models under DP to generate synthetic text, and have shown that downstream models trained on such synthetic data can sometimes approach the performance of models trained with DP directly on real data~\cite{mattern2022differentially, yue2023synthetic, kurakin2023harnessing}. However, achieving a favorable fidelity--privacy trade-off through direct DP training can require large batch sizes and long training schedules~\cite{anil2022large}, which makes it costly for large language models.

More recent work therefore considers API-based generation, where strong pre-trained models are queried with non-sensitive prompts and DP is applied only to the data-dependent selection or adaptation stage~\cite{lin2023differentially, xie2024differentially}. This line of work is particularly relevant to our setting because it separates high-capability text generation from privacy-sensitive operations, allowing the private data to remain local while still benefiting from strong generative models.

In our implementation, we instantiate the synthetic data construction stage using a recent API-based DP synthetic text generation method~\cite{xie2024differentially}. We emphasize that this component is not tied to a specific generator. \ourmethod{} only requires a DP synthetic dataset that preserves enough task-relevant structure for layer selection. Therefore, stronger future DP synthetic data generators can be naturally plugged into \ourmethod{} and may further improve its performance.

\clearpage

\onecolumn

\section{Theoretical Results}
\label{app:theory}

In this appendix, we provide the formal statements and proofs of the theoretical results discussed in \autoref{subsec:analysis}.

\subsection{Proof of Theorem 1}
\label{app:proof_thm1}

We first formalize the one-step effect of DP noise under selective fine-tuning.

\begin{theorem}[Selective fine-tuning trades learning signal for reduced DP noise damage]
\label{thm:selection_signal_damage}
Let $\Lambda\subseteq[L]$ be a selected layer subset with $d_{\Lambda}$ trainable parameters. Assume that $F_{\mathrm{pri}}(\theta)$ is $\beta$-smooth with respect to $\theta_{\Lambda}$. Consider one DP update restricted to $\Lambda$:
\begin{equation}
\theta_{\Lambda}^{+}
=
\theta_{\Lambda}
-
\eta(g_{\Lambda}+z_{\Lambda}),
\qquad
\|g_{\Lambda}\|_2\leq C,
\qquad
z_{\Lambda}\sim \mathcal{N}(0,\sigma^2 C^2 I_{d_{\Lambda}}).
\end{equation}
Then,
\begin{equation}
\mathbb{E}
\left[
F_{\mathrm{pri}}(\theta^{+})
\right]
\leq
F_{\mathrm{pri}}(\theta)
-
\eta
\left\langle
\nabla_{\theta_{\Lambda}}F_{\mathrm{pri}}(\theta),
g_{\Lambda}
\right\rangle
+
\frac{\beta\eta^2}{2}\|g_{\Lambda}\|_2^2
+
\frac{\beta\eta^2}{2}d_{\Lambda}\sigma^2 C^2.
\end{equation}
In particular, if $g_{\Lambda}$ is aligned with the projected gradient and satisfies
\begin{equation}
\left\langle
\nabla_{\theta_{\Lambda}}F_{\mathrm{pri}}(\theta),
g_{\Lambda}
\right\rangle
\geq
a_{\Lambda}
\left\|
\nabla_{\theta_{\Lambda}}F_{\mathrm{pri}}(\theta)
\right\|_2^2
\end{equation}
for some alignment coefficient $a_{\Lambda}>0$, then
\begin{equation}
\mathbb{E}
\left[
F_{\mathrm{pri}}(\theta^{+})
\right]
\leq
F_{\mathrm{pri}}(\theta)
-
\eta a_{\Lambda}
\left\|
\nabla_{\theta_{\Lambda}}F_{\mathrm{pri}}(\theta)
\right\|_2^2
+
\frac{\beta\eta^2}{2}C^2
+
\frac{\beta\eta^2}{2}d_{\Lambda}\sigma^2 C^2.
\end{equation}
Moreover, defining the signal retention ratio as
\begin{equation}
s(\Lambda)
=
\frac{
\left\|
\nabla_{\theta_{\Lambda}}F_{\mathrm{pri}}(\theta)
\right\|_2^2
}{
\left\|
\nabla_{\theta}F_{\mathrm{pri}}(\theta)
\right\|_2^2
},
\end{equation}
the bound can be written as
\begin{equation}
\mathbb{E}
\left[
F_{\mathrm{pri}}(\theta^{+})
\right]
\leq
F_{\mathrm{pri}}(\theta)
-
\eta a_{\Lambda}
s(\Lambda)
\left\|
\nabla_{\theta}F_{\mathrm{pri}}(\theta)
\right\|_2^2
+
\frac{\beta\eta^2}{2}C^2
+
\frac{\beta\eta^2}{2}d_{\Lambda}\sigma^2 C^2.
\end{equation}
Thus, the expected one-step progress is governed by a trade-off between the retained learning signal $s(\Lambda)$ and the DP-noise damage term $d_{\Lambda}\sigma^2 C^2$.
\end{theorem}

\begin{proof}
Let
\begin{equation}
\Delta_{\Lambda}
=
-\eta(g_{\Lambda}+z_{\Lambda}).
\end{equation}
By $\beta$-smoothness of $F_{\mathrm{pri}}$ with respect to $\theta_{\Lambda}$, we have
\begin{equation}
\begin{aligned}
F_{\mathrm{pri}}(\theta^{+})
&\leq
F_{\mathrm{pri}}(\theta)
+
\left\langle
\nabla_{\theta_{\Lambda}}F_{\mathrm{pri}}(\theta),
\Delta_{\Lambda}
\right\rangle
+
\frac{\beta}{2}
\|\Delta_{\Lambda}\|_2^2 \\
&=
F_{\mathrm{pri}}(\theta)
-
\eta
\left\langle
\nabla_{\theta_{\Lambda}}F_{\mathrm{pri}}(\theta),
g_{\Lambda}+z_{\Lambda}
\right\rangle
+
\frac{\beta\eta^2}{2}
\|g_{\Lambda}+z_{\Lambda}\|_2^2 .
\end{aligned}
\end{equation}
Taking expectation over $z_{\Lambda}$ and using $\mathbb{E}[z_{\Lambda}]=0$, we obtain
\begin{equation}
\begin{aligned}
\mathbb{E}
\left[
F_{\mathrm{pri}}(\theta^{+})
\right]
&\leq
F_{\mathrm{pri}}(\theta)
-
\eta
\left\langle
\nabla_{\theta_{\Lambda}}F_{\mathrm{pri}}(\theta),
g_{\Lambda}
\right\rangle
+
\frac{\beta\eta^2}{2}
\mathbb{E}\|g_{\Lambda}+z_{\Lambda}\|_2^2 \\
&=
F_{\mathrm{pri}}(\theta)
-
\eta
\left\langle
\nabla_{\theta_{\Lambda}}F_{\mathrm{pri}}(\theta),
g_{\Lambda}
\right\rangle
+
\frac{\beta\eta^2}{2}
\left(
\|g_{\Lambda}\|_2^2
+
\mathbb{E}\|z_{\Lambda}\|_2^2
\right) \\
&=
F_{\mathrm{pri}}(\theta)
-
\eta
\left\langle
\nabla_{\theta_{\Lambda}}F_{\mathrm{pri}}(\theta),
g_{\Lambda}
\right\rangle
+
\frac{\beta\eta^2}{2}\|g_{\Lambda}\|_2^2
+
\frac{\beta\eta^2}{2}d_{\Lambda}\sigma^2 C^2 .
\end{aligned}
\end{equation}
This proves the first claim.

If the alignment condition holds, then
\begin{equation}
-
\eta
\left\langle
\nabla_{\theta_{\Lambda}}F_{\mathrm{pri}}(\theta),
g_{\Lambda}
\right\rangle
\leq
-
\eta a_{\Lambda}
\left\|
\nabla_{\theta_{\Lambda}}F_{\mathrm{pri}}(\theta)
\right\|_2^2.
\end{equation}
Using $\|g_{\Lambda}\|_2\leq C$, we obtain
\begin{equation}
\mathbb{E}
\left[
F_{\mathrm{pri}}(\theta^{+})
\right]
\leq
F_{\mathrm{pri}}(\theta)
-
\eta a_{\Lambda}
\left\|
\nabla_{\theta_{\Lambda}}F_{\mathrm{pri}}(\theta)
\right\|_2^2
+
\frac{\beta\eta^2}{2}C^2
+
\frac{\beta\eta^2}{2}d_{\Lambda}\sigma^2 C^2.
\end{equation}
Finally, by the definition of $s(\Lambda)$,
\begin{equation}
\left\|
\nabla_{\theta_{\Lambda}}F_{\mathrm{pri}}(\theta)
\right\|_2^2
=
s(\Lambda)
\left\|
\nabla_{\theta}F_{\mathrm{pri}}(\theta)
\right\|_2^2.
\end{equation}
Substituting this identity into the previous inequality gives
\begin{equation}
\mathbb{E}
\left[
F_{\mathrm{pri}}(\theta^{+})
\right]
\leq
F_{\mathrm{pri}}(\theta)
-
\eta a_{\Lambda}
s(\Lambda)
\left\|
\nabla_{\theta}F_{\mathrm{pri}}(\theta)
\right\|_2^2
+
\frac{\beta\eta^2}{2}C^2
+
\frac{\beta\eta^2}{2}d_{\Lambda}\sigma^2 C^2.
\end{equation}
This completes the proof.
\end{proof}

\subsection{Proof of Theorem 2}
\label{app:proof_thm2}

We next formalize why worst-case perturbation provides a robust selection criterion for unseen DP noise.

\begin{theorem}[Worst-case selection generalizes to unseen DP noise]
\label{thm:worst_case_generalization}
Let $\mathcal{Q}$ be the family of candidate layer subsets. For each $\Lambda\in\mathcal{Q}$, define
\begin{equation}
W_{\mathrm{syn}}(\Lambda)
=
\sup_{\|\xi\|_2\leq \rho}
\mathcal{R}_{\mathrm{syn}}(\Lambda,\xi),
\qquad
\Lambda^{\star}
=
\arg\min_{\Lambda\in\mathcal{Q}}
W_{\mathrm{syn}}(\Lambda).
\end{equation}
Assume that, for all $\Lambda\in\mathcal{Q}$ and all $\|\xi\|_2\leq\rho$,
\begin{equation}
\left|
\mathcal{R}_{\mathrm{syn}}(\Lambda,\xi)
-
\mathcal{R}_{\mathrm{pri}}(\Lambda,\xi)
\right|
\leq
\tau,
\qquad
0
\leq
\mathcal{R}_{\mathrm{pri}}(\Lambda,\xi)
\leq
B.
\end{equation}
If the downstream DP perturbation $Z$ satisfies
\begin{equation}
\Pr(\|Z\|_2\leq\rho)\geq 1-\alpha,
\end{equation}
then
\begin{equation}
\mathbb{E}_{Z}
\left[
\mathcal{R}_{\mathrm{pri}}(\Lambda^{\star},Z)
\right]
\leq
\min_{\Lambda\in\mathcal{Q}}
\sup_{\|\xi\|_2\leq\rho}
\mathcal{R}_{\mathrm{pri}}(\Lambda,\xi)
+
2\tau
+
\alpha B.
\end{equation}
\end{theorem}

\begin{proof}
Define
\begin{equation}
W_{\mathrm{pri}}(\Lambda)
=
\sup_{\|\xi\|_2\leq\rho}
\mathcal{R}_{\mathrm{pri}}(\Lambda,\xi).
\end{equation}
By the uniform approximation assumption, for every $\Lambda\in\mathcal{Q}$,
\begin{equation}
W_{\mathrm{pri}}(\Lambda)
\leq
W_{\mathrm{syn}}(\Lambda)+\tau,
\qquad
W_{\mathrm{syn}}(\Lambda)
\leq
W_{\mathrm{pri}}(\Lambda)+\tau.
\end{equation}
Since $\Lambda^\star$ minimizes $W_{\mathrm{syn}}$, for any $\Lambda\in\mathcal{Q}$,
\begin{equation}
\begin{aligned}
W_{\mathrm{pri}}(\Lambda^\star)
&\leq
W_{\mathrm{syn}}(\Lambda^\star)+\tau \\
&\leq
W_{\mathrm{syn}}(\Lambda)+\tau \\
&\leq
W_{\mathrm{pri}}(\Lambda)+2\tau.
\end{aligned}
\end{equation}
Taking the minimum over $\Lambda\in\mathcal{Q}$ gives
\begin{equation}
\sup_{\|\xi\|_2\leq\rho}
\mathcal{R}_{\mathrm{pri}}(\Lambda^{\star},\xi)
\leq
\min_{\Lambda\in\mathcal{Q}}
\sup_{\|\xi\|_2\leq\rho}
\mathcal{R}_{\mathrm{pri}}(\Lambda,\xi)
+
2\tau.
\end{equation}

Now decompose the expectation over the event $\{\|Z\|_2\leq\rho\}$:
\begin{equation}
\begin{aligned}
\mathbb{E}_{Z}
\left[
\mathcal{R}_{\mathrm{pri}}(\Lambda^{\star},Z)
\right]
&=
\mathbb{E}_{Z}
\left[
\mathcal{R}_{\mathrm{pri}}(\Lambda^{\star},Z)
\mathbf{1}\{\|Z\|_2\leq\rho\}
\right]
+
\mathbb{E}_{Z}
\left[
\mathcal{R}_{\mathrm{pri}}(\Lambda^{\star},Z)
\mathbf{1}\{\|Z\|_2>\rho\}
\right] \\
&\leq
\sup_{\|\xi\|_2\leq\rho}
\mathcal{R}_{\mathrm{pri}}(\Lambda^{\star},\xi)
+
B\Pr(\|Z\|_2>\rho) \\
&\leq
\sup_{\|\xi\|_2\leq\rho}
\mathcal{R}_{\mathrm{pri}}(\Lambda^{\star},\xi)
+
\alpha B.
\end{aligned}
\end{equation}
Combining the two bounds yields
\begin{equation}
\mathbb{E}_{Z}
\left[
\mathcal{R}_{\mathrm{pri}}(\Lambda^{\star},Z)
\right]
\leq
\min_{\Lambda\in\mathcal{Q}}
\sup_{\|\xi\|_2\leq\rho}
\mathcal{R}_{\mathrm{pri}}(\Lambda,\xi)
+
2\tau
+
\alpha B.
\end{equation}
This completes the proof.
\end{proof}

\end{document}